\documentclass{article}
\usepackage{iclr2026_conference,times}

\iclrfinalcopy

\usepackage{hyperref}
\usepackage{url}
\usepackage{graphicx}
\usepackage{amsmath}
\usepackage{amssymb}
\usepackage{amsthm}
\usepackage{booktabs}
\usepackage{algorithm}
\usepackage{algorithmic}
\usepackage{multirow}
\usepackage[scr=rsfs]{mathalpha}
\usepackage{svg}
\usepackage{tikz}
\usepackage{wrapfig}
\usepackage{tcolorbox}
\usepackage{graphicx}

\newtheorem{theorem}{Theorem}

\newtheorem{proposition}[theorem]{Proposition}
\newtheorem{corollary}[theorem]{Corollary}
\newtheorem{remark}[theorem]{Remark}

\title{Clarification as Supervision: Reinforcement Learning for Vision-Language Interfaces}

\title{Clarification as Supervision: Reinforcement Learning for Vision-Language Interfaces}

\makeatletter
\@ifpackageloaded{microtype}{\microtypesetup{protrusion=false}}{}
\makeatother
\author{
  \hspace{-0.08cm}John Gkountouras \\
  ILLC, University of Amsterdam \\
  \texttt{i.gkountouras@uva.nl}
  \And
  Ivan Titov \\
  ILLC, University of Amsterdam \\
  ILCC, University of Edinburgh \\
  \texttt{ititov@inf.ed.ac.uk}
}
\makeatletter
\@ifpackageloaded{microtype}{\microtypesetup{protrusion=true}}{}
\makeatother

\begin{document}
\maketitle

\begin{abstract}
Recent text-only models demonstrate remarkable mathematical reasoning capabilities. Extending these to visual domains requires vision-language models to translate images into text descriptions. However, current models, trained to produce captions for human readers, often omit the precise details that reasoning systems require.
This creates an interface mismatch: reasoners often fail not due to reasoning limitations but because they lack access to critical visual information.
We propose Adaptive-Clarification Reinforcement Learning (AC-RL), which teaches vision models what information reasoners need through interaction. Our key insight is that clarification requests during training reveal information gaps; by penalizing success that requires clarification, we create pressure for comprehensive initial captions that enable the reasoner to solve the problem in a single pass.
AC-RL improves average accuracy by $4.4$ points over pretrained baselines across seven visual mathematical reasoning benchmarks, and analysis shows it would cut clarification requests by up to $39$\% if those were allowed.
By treating clarification as a form of implicit supervision, AC-RL demonstrates that vision-language interfaces can be effectively learned through interaction alone, without requiring explicit annotations.
\end{abstract}

\section{Introduction}

Recent advances in reinforcement learning have produced text-based reasoning models with remarkable mathematical capabilities~\citep{deepseek_r1, deepseek_math}. 
While these reasoning capabilities are impressive, extending them to visual domains requires careful consideration of how visual and linguistic information should interface.

Several recent works explore decoupled architectures for visual reasoning, where vision modules translate images into text descriptions that are then processed by text-only reasoners~\citep{cola, vicor, visprog}. This modular paradigm offers practical advantages: it enables reuse of existing text-only reasoning models without costly multimodal retraining, allows flexible composition of specialized components, and provides interpretable interfaces between perception and reasoning. Systems like COLA, ViCor, and VCTP demonstrate this approach, using LLMs as coordinators or reasoners that operate on text descriptions of visual content~\citep{vctp}. The common thread is that visual information flows through a linguistic bottleneck, requiring careful design of what information to communicate~\citep{beyond_captioning, integrating_visual_interpretation}.

However, this decoupling creates a critical alignment challenge: vision-language tools must learn what visual information each specific reasoner requires for successful problem-solving. Vision-language models are typically trained on diverse multimodal datasets to produce descriptions sufficient for general visual understanding and question answering. Yet different reasoning models may have distinct information needs: one might excel with precise measurements, while another benefits from structural or topological descriptions. Traditional supervised approaches would require annotating ``ideal captions'' for each reasoner, an infeasible task given the diversity of visual reasoning problems and the implicit nature of reasoner preferences. Moreover, what constitutes an informative caption cannot be determined a priori; it emerges only through interaction with the reasoning model.

Reinforcement learning offers a natural framework for this interface learning problem, but applying it to vision-reasoner coordination is challenging. The primary difficulty lies in the sparsity of learning signals: when using binary task rewards, the vision model receives identical zero rewards whether its caption is completely inadequate or missing one crucial detail. Additionally, the large action space of language generation combined with sparse rewards leads to inefficient exploration, where most generated captions result in failure without providing informative gradients. These challenges are compounded by the indirect nature of the feedback: the vision model must infer what information the reasoner needed based solely on binary success signals, without explicit guidance about what was missing.

To overcome the limitations of this sparse reward signal, we introduce Adaptive-Clarification Reinforcement Learning (AC-RL), a framework that enables vision-language models to learn effective interfaces with specific reasoners through a clarification-aware training scaffold (Figure~\ref{fig:figure1}). The key insight is twofold: clarification dialogues during training reveal information gaps that would otherwise result in zero reward, and by penalizing clarification-dependent success, we create optimization pressure for comprehensive initial captions. During training, when the reasoner requests clarification, a frozen reference model responds, ensuring gradients flow only through the initial caption. This design teaches the vision model to front-load relevant information, eliminating the need for costly clarifications at inference, where the system operates in a single pass.

\begin{figure}[!t]
    \vspace{-0.5cm}
    \centering
    \includegraphics[width=\textwidth]{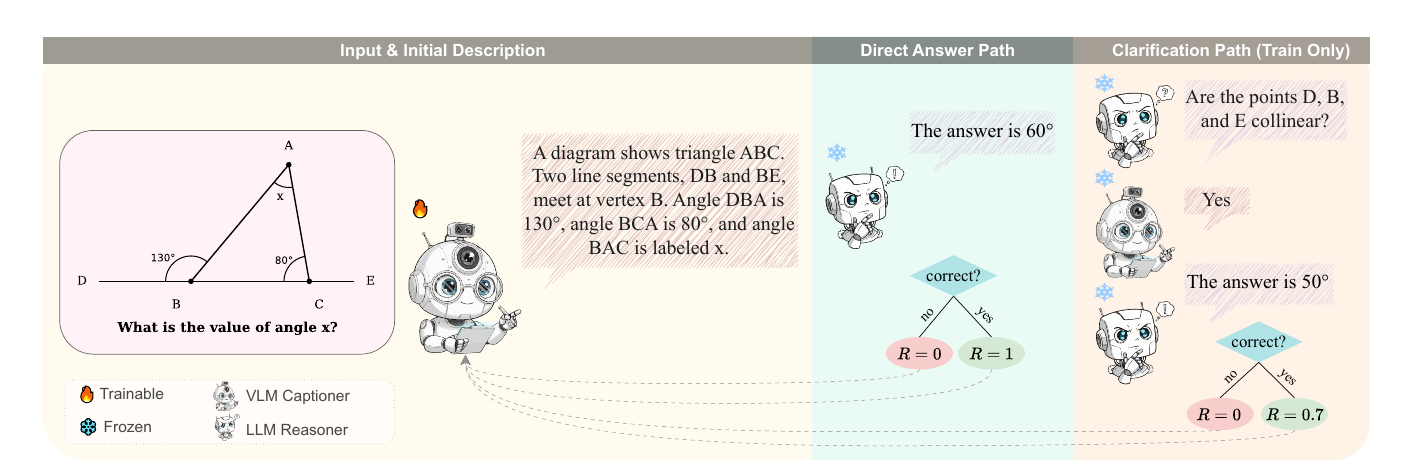}
    \vspace{-0.5cm}
    \caption{\textbf{Adaptive-Clarification Reinforcement Learning (AC-RL) training framework}. 
    Given an image and a question, a trainable captioner generates an initial description. 
    During training, the frozen reasoner evaluates whether this description contains sufficient detail to solve the problem. 
    If yes (Direct Answer Path), it attempts to answer directly, receiving reward $R=1$ for correct answers or $R=0$ for incorrect ones. 
    If the description lacks crucial information (Clarification Path), the reasoner requests specific details, which are provided by a frozen reference captioner. 
    Correct answers after clarification receive partial reward $R=0.7$, while incorrect answers receive $R=0$.
    Gradients (dotted arrows) flow only through the initial caption generation, not through clarification responses.
    At inference, only the direct answer path is used: the model has learned to generate sufficiently detailed initial captions, eliminating the need for clarification.}
    \label{fig:figure1}
    \vspace{-0.5cm}
\end{figure}

More specifically, AC-RL transforms sparse binary rewards into a tiered structure: full reward for direct success, partial reward (we used $\alpha=0.7$) for success requiring clarification, and zero for failure. This densification serves dual purposes. First, it converts many zero-reward episodes into partially rewarded ones, providing a gradient signal when initial captions are nearly sufficient. Second, the penalty $(1-\alpha)$ creates pressure to discover self-sufficient captioning strategies aligned with single-pass deployment. Through thousands of interactions, the vision model explores different description strategies, learning without explicit supervision what quantitative details, spatial relationships, or structural patterns this particular reasoner needs for mathematical problem-solving.

The key contributions of our work are as follows:
\begin{itemize}
    \item An exploration-based framework that enables vision-language models to discover through reinforcement learning what visual information a reasoner requires, adapting from human-caption pretraining without explicit supervision.
    \item A clarification-aware reward structure that uses interaction patterns as learning signals, allowing models to identify information gaps and iteratively improve their captioning strategies through trial and error.
    \item An empirical demonstration that our clarification-aware training scaffold effectively teaches captioners to anticipate reasoner needs, leading to improved accuracy on seven mathematical VQA benchmarks and a measurable reduction in clarification dependency at inference.
\end{itemize}

\section{Related Work}

\paragraph{Reinforcement learning for reasoning in language models.}
Recent work shows RL can teach extended mathematical reasoning, with DeepSeek-R1 demonstrating learned policies outperform prompt-based chain-of-thought \citep{deepseek_r1}. Visual extensions employ diverse strategies: training stability (Skywork-R1V2 \citep{skywork_r1v2}, Vision-R1 \citep{vision-r1}), replay mechanisms (VL-Rethinker \citep{vl_rethinker}, OpenVLThinker \citep{openvlthinker}), and cross-modal formalization (R1-OneVision \citep{r1-onevision}, Mulberry \citep{mulberry}). These methods focus on extending reasoning chains to handle visual inputs. We take an orthogonal approach by shaping the interface between perception and reasoning modules, using clarification-aware rewards to teach captioners what information reasoners need rather than how to reason about it.

\paragraph{Decoupled perception–reasoning and interface design}
Decoupling visual perception from linguistic reasoning offers modularity and the ability to reuse strong text-only reasoners, but it raises an interface-alignment challenge: captions optimized for human readability may omit the quantitative and structural cues a reasoner needs \citep{vicor, integrating_visual_interpretation, beyond_captioning}. Coordination frameworks use an LLM to route or aggregate information from one or more VLMs (e.g., COLA’s coordinator that queries complementary experts) \citep{cola}, or to interleave ``see-think-confirm’’ phases that explicitly ground and verify intermediate steps (VCTP) \citep{vctp}. Neuro-symbolic systems like VisProg sidestep monolithic pipelines by composing programs over off-the-shelf vision tools~\citep{visprog}. Our approach adheres to the decoupled setup but replaces fixed protocols with an RL objective that learns, from interaction, which \textit{caption features} best serve a specific reasoner.

\paragraph{Learning alignment through interaction.}
Several methods optimize captions specifically for reasoning rather than human readability. Most relevant to our work, RACRO directly uses binary task rewards to align a captioner to a reasoner \citep{racro}, demonstrating that interface learning is possible through RL alone. However, RACRO relies solely on sparse binary rewards, which we show can be significantly improved through our clarification-aware tiered reward structure that densifies the learning signal. LAMOC and VLRM leverage language model feedback and VLM-as-reward-model, respectively \citep{lamoc, vlrm}. OmniCaptioner generates long-context descriptions that improve LLM reasoning \citep{omnicaptioner}, while Critic-V employs a learned VLM critic \citep{critic-v}. Beyond vision, multi-agent frameworks have shown that LLMs can coordinate through language-only protocols, and that adapting inputs to a solver's biases can improve performance~\citep{wu2024autogen, paraphrase_solve}.

\section{Methodology}
\label{sec:methodology}
\subsection{The Vision-Reasoner Interface Problem}

We consider a modular architecture where a trainable vision-language model, the captioner $\mathcal{V}_\theta$, translates images into text descriptions that enable a frozen text-only model, the reasoner $\mathcal{R}$, to solve visual reasoning tasks. Given an image $I$ and question $Q$, the system produces an answer $A$. The central challenge lies in learning what visual information the specific reasoner requires, without explicit supervision defining ``ideal captions''.

Our approach leverages clarification dialogues as implicit supervision. When the reasoner requests additional information during training, it reveals that the initial caption failed to communicate adequately. Adaptive-Clarification Reinforcement Learning (AC-RL) exploits this signal through a tiered reward structure and clarification-aware training that teaches the captioner to anticipate and provide information the reasoner would request.

\subsection{Training and Inference Protocols}

During training, we permit structured interaction between the captioner and reasoner. The captioner first generates an initial caption $c_0 \sim \pi_\theta(\cdot \mid I, Q)$ describing the visual content. The reasoner processes this caption and either produces an answer directly or requests clarification with a specific question $q_1$. When clarification is requested, a \emph{frozen reference model} $\pi_{\text{ref}}$ provides the response $c_1 \sim \pi_{\text{ref}}(\cdot \mid I, Q, q_1)$. The reasoner then produces its final answer $A$ using all available information.

Crucially, the clarification response comes from a frozen checkpoint that receives no gradients during training. This design ensures that the captioner cannot rely on improving clarification capabilities and must instead learn to front-load relevant information into the initial caption. Details of this protocol and the complete algorithm appear in Figure~\ref{fig:figure1} and Appendix~\ref{app:algorithm}.

At inference time, the system operates in a single pass: the captioner generates one description $c_0 \sim \pi_\theta(\cdot \mid I, Q)$, and the reasoner must produce the answer based solely on this initial caption. This single-pass constraint is crucial for practical applications where multi-turn interaction would be computationally expensive or require architectural changes to existing tool-calling frameworks. By learning to front-load information during training, our approach produces captioners that work with standard single-pass inference. The reasoner processes the initial caption without needing to be modified to request clarifications.

\subsection{Clarification-Aware Reward Design}

A key contribution of AC-RL is the tiered reward structure that densifies the learning signal. In standard reinforcement learning for visual question answering, episodes receive binary rewards based solely on answer correctness. This sparse signal provides limited feedback when the captioner produces nearly sufficient but incomplete descriptions.

Our reward function addresses this sparsity by distinguishing three outcomes:
\begin{equation}
R(\tau) = \begin{cases}
1 & \text{if correct answer without clarification} \\
\alpha & \text{if correct answer with clarification} \\
0 & \text{if incorrect answer}
\end{cases}
\end{equation}
where $\alpha \in (0, 1)$ and $\tau$ denotes the complete episode trajectory. We set $\alpha = 0.7$.

This structure serves dual purposes. First, it converts many zero-reward episodes into partially rewarded ones, providing gradient signal when the initial caption contains most but not all necessary information. This densification is particularly valuable early in training when captions frequently lack specific details. Second, the penalty $(1 - \alpha)$ for requiring clarification creates optimization pressure toward self-sufficient initial captions that align with single-pass deployment.

The clarification mechanism thus acts as a scaffold that provides intermediate credit assignment. Episodes where the reasoner would fail with the initial caption alone but succeeds after clarification receive partial reward, signaling that the caption was nearly adequate. This graded feedback enables more sample-efficient learning compared to binary rewards that treat all failures equivalently.

\subsection{Policy Optimization}

We optimize the captioner using a KL-regularized objective that balances task performance with proximity to the pretrained initialization:
\begin{equation}
J(\theta) = \mathbb{E}_{(I,Q) \sim \mathcal{D}} \left[ \mathbb{E}_{\tau \sim \pi_\theta}[R(\tau)] \right] - \beta \cdot D_{\text{KL}}(\pi_\theta \| \pi_{\text{ref}})
\end{equation}
where $\pi_{\text{ref}}$ denotes a fixed reference policy for regularization.

We employ Beta-Normalization Policy Optimization~\citep{bnpo} for advantage estimation, as it provides stable normalization. Importantly, gradients flow only through the initial caption generation $c_0$. Neither the frozen reasoner $\mathcal{R}$ nor the clarification model $\pi_{\text{ref}}$ receive gradient updates, ensuring the captioner adapts unilaterally to the fixed reasoner's preferences. The complete algorithmic specification appears in Appendix~\ref{app:algorithm}.

\section{Experiments}
\label{sec:experiments}
We evaluate whether Adaptive-Clarification Reinforcement Learning (AC-RL) successfully aligns vision-language models with the information needs of downstream reasoning systems. Our experimental design tests three key hypotheses: (1) AC-RL improves task performance compared to both pretrained models and standard reinforcement learning approaches (i.e., learning with the tiered rewards and clarifications is beneficial), (2) the clarification-aware training scaffold contributes meaningfully to performance gains beyond standard RL, and (3) the improvements stem from learning to front-load reasoner-relevant information into initial captions.

\label{subsec:setup}

\paragraph{System Architecture.}
We instantiate the trainable captioning policy $\mathcal{V}_\theta$ with InternVL3-2B or Qwen2.5-VL-3B. Their modest size enables extensive RL experimentation and ablations, and when paired with a strong reasoner, they provide reliable baseline competence across the evaluated benchmarks. Their moderate scale also makes GRPO-style optimization tractable without requiring extensive computational resources. The frozen reasoning system $\mathcal{R}$ is DeepSeek-R1-Qwen-32B, a powerful text-only model trained for mathematical reasoning with particularly strong instruction following capabilities. We also evaluate the vision models as standalone systems to quantify the benefits of architectural decoupling.

\paragraph{Training Configurations and Method Baselines}
We compare four training configurations to isolate the effects of different design choices. The \textbf{Standalone VLM} baseline has the vision-language model answer questions directly without a separate reasoner. The \textbf{Pretrained + Reasoner} configuration pairs the pretrained VLM with the frozen reasoner without fine-tuning, measuring the immediate benefit of modular architectures. \textbf{Binary-Reward RL} fine-tunes the captioner with binary task success rewards, similarly to recent work, RACRO~\citep{racro}. Finally, \textbf{AC-RL} employs our tiered rewards and clarification-aware training scaffold. These baselines allow us to decompose gains from architectural decoupling, reinforcement learning, and our clarification-aware training scaffold. All RL methods are trained on ViRL-39K~\citep{vl_rethinker}, a visual instruction dataset focused on mathematical reasoning, with evaluation performed on separate held-out benchmarks.

\paragraph{Training Protocol.}
AC-RL training uses the clarification-aware scaffold detailed in Section \ref{sec:methodology}. During training, the captioner generates $c_0 \sim \pi_\theta$, and if the reasoner requests clarification, a frozen reference policy provides the response. The tiered rewards ($R=1$ for direct success, $R=0.7$ with clarification, $R=0$ for failure) create gradients only through the initial caption. We optimize using BNPO with KL regularization. Notably, AC-RL maintains greater generation diversity than standard RL throughout training (Appendix~\ref{app:diversity}).

\paragraph{Evaluation Protocol.}
All models are evaluated using \textbf{single-pass evaluation}: the captioner produces a description that the reasoner uses to generate a final answer, with no clarification permitted. This protocol ensures that performance gains reflect improved caption quality rather than multi-turn interaction benefits. For behavioral analyses in Section \ref{subsec:mechanism}, we additionally conduct instrumented runs with \textbf{clarification-enabled evaluation} where clarification is allowed, to measure clarification patterns.

\paragraph{Datasets and Baselines}
We evaluate on seven mathematical reasoning benchmarks: MathVista (testmini)~\citep{mathvista}, MathVision~\citep{mathvision}, MathVerse~\citep{mathverse}, MMMU (validation)~\citep{mmmu}, WeMath (strict)~\citep{wemath}, DynaMath (worst-case)~\citep{dynamath}, and LogicVista~\citep{logicvista}. We report exact-match accuracy using EvalScope~\citep{evalscope_2024} and VLMEvalKit~\citep{vlmevalkit}. We compare against leading proprietary models: GPT-4o~\citep{gpt4}, Claude-3.7-Sonnet, Gemini-2.0-Flash~\citep{gemini}, o1~\citep{openai_o1}, Gemini 2.5 Pro~\citep{gemini2.5}, and Seed1.5-VL (Thinking)~\citep{seed15vl}; open-weights general-purpose models: Qwen2.5-VL~\citep{qwen25vl}, and Ovis2~\citep{ovis}; and reasoning-optimized models: InternVL3-MPO variants~\citep{internvl3}, VL-Rethinker~\citep{vl_rethinker}, QVQ-72B-Preview~\citep{qvq}, and MMR1~\citep{MMR1}.
\subsection{Overall Performance}
\label{subsec:mainresults}

Table~\ref{tab:main_results} presents our results in the context of leading proprietary and open-weights models. We first note that small vision-language models achieve limited performance when solving problems directly: InternVL-2B and Qwen-3B reach only 32.4\% and 34.6\% average accuracy respectively as standalone systems. Simply pairing these models with a strong reasoner (Pretrained + Reasoner) improves performance to 39.3\% and 39.0\%, demonstrating the value of modular architectures. However, applying AC-RL yields the most substantial gains.

With a Qwen-3B captioner, AC-RL improves the average accuracy from 39.0 to 43.4 (+4.4 points), with substantial gains on robustness and vision-centric benchmarks like \emph{DynaMath} (+10.6) and \emph{MathVerse} (+5.2). The InternVL-2B captioner sees a similar +3.3 average point increase. These results, obtained under an identical single-pass protocol, demonstrate that AC-RL effectively aligns the captioning policy with the downstream reasoner's needs.

These gains confirm that AC-RL successfully modifies the captioning policy to be more effective for the downstream reasoner. While we observe minor regressions on WeMath, the broad improvements across other benchmarks suggest that AC-RL effectively learns to prioritize the quantitative and structural details essential for complex visual reasoning.

\begin{table*}[!th]
\centering
\caption{Main results on multi-modal reasoning benchmarks: MathVista (MVista), MathVision (MVision), MathVerse (MVerse), MMMU, WeMath (WeM), DynaMath (DynaM), and LogicVista (LVista). Our AC-RL method, evaluated in the final blocks for each model size, significantly enhances the performance of small vision models.}
\vspace{0.25cm}
\label{tab:main_results}
\resizebox{\textwidth}{!}{%
\begin{tabular}{@{}l|ccccccc|c@{}}
\toprule
\textbf{Model} & \textbf{\footnotesize{MVista}} & \textbf{\footnotesize{MVision}} & \textbf{\footnotesize{MVerse}} & \textbf{\footnotesize{MMMU}} & \textbf{\footnotesize{WeM}} & \textbf{\footnotesize{DynaM}} & \textbf{\footnotesize{LVista}} & \textbf{\footnotesize{AVG}} \\
\midrule
\addlinespace[0.35em]
\multicolumn{9}{@{}l}{\textbf{\textsc{Proprietary Models}}} \\
\addlinespace[0.15em]
GPT-4o-20241120   & 60.0 & 31.2 & 40.6 & 70.7 & 45.8 & 34.5 & 52.8 & 47.9 \\
Gemini-2.0-Flash  & 70.4 & 43.6 & 47.7 & 72.6 & 47.4 & 42.1 & 52.3 & 53.7 \\
Claude-3.7-Sonnet & 66.8 & 41.9 & 46.7 & 75.0 & 49.3 & 39.7 & 58.2 & 53.9 \\
o1                & 73.9 & 42.2 & ---    & \textbf{78.2} & ---    & ---    & ---    & --- \\
Gemini 2.5 Pro    & \textbf{80.9} & \textbf{69.1} & \textbf{76.9} & 74.7 & \textbf{78.0} & \textbf{56.3} & \textbf{73.8} & \textbf{72.8} \\
Seed1.5-VL (Thinking) & 79.5$^\dagger$ & 68.7 & ---    & 77.9 & 77.5 & ---    & ---    & 75.9$^*$ \\
\midrule
\addlinespace[0.35em]
\multicolumn{9}{@{}l}{\textbf{\textsc{Open-Weights Models}}} \\
\addlinespace[0.15em]
InternVL3-2B-MPO  & 57.0 & 21.7 & 25.3 & 48.6 & 22.4 & 14.6 & 36.9 & 32.4 \\
InternVL3-8B-MPO  & 71.6 & 29.3 & 39.8 & 62.7 & 37.1 & 25.5 & 44.1 & 44.3 \\
Ovis2-8B          & 71.8$^\dagger$ & 25.9 & 42.3 & 59.0 & ---    & ---    & 39.4 & 47.7 \\
InternVL3-14B-MPO & 75.1 & 37.2 & 44.4 & 67.1 & 43.0 & 31.3 & 51.2 & 49.9 \\
QVQ-72B-Preview   & 70.3 & 34.9 & 48.2 & 70.3 & 39.0 & 30.7 & 58.2 & 50.2 \\
MMR1-Math-v0-7B   & 71.0$^\dagger$ & 30.2 & 49.2 & ---    & ---    & ---    & 50.8 & 50.3 \\
InternVL3-38B-MPO & 75.1 & 34.2 & 48.2 & 70.1 & \textbf{48.6} & \textbf{35.3} & \textbf{58.4} & 52.8 \\
VL-Rethinker-72B  & \textbf{80.3} & \textbf{43.9} & ---    & 68.8 & ---    & ---    & ---    & --- \\
InternVL3-78B-MPO & 79.0 & 43.1 & \textbf{51.0} & \textbf{72.2} & 46.0 & 35.1 & 55.9 & \textbf{54.6} \\
\midrule
\addlinespace[0.35em]
\multicolumn{9}{@{}l}{\textbf{\textsc{InternVL-2B}}} \\
\addlinespace[0.15em]
Standalone VLM & 57.0 & 21.9 & 25.3 & 48.6 & 22.4 & 14.6 & 36.9 & 32.4 \\
Pretrained + Reasoner & 61.0 & 34.7 & 28.9 & 57.4 & \textbf{32.8} & 12.0 & 48.3 & 39.3 \\
\textbf{AC-RL (ours)} 
& \textbf{65.3} & \textbf{36.7} & \textbf{36.8} & \textbf{58.4} & 32.1 & \textbf{20.0} & \textbf{49.0} & \textbf{42.6} \\
\phantom{\textbf{AC-RL (ours)}}
& \textcolor{teal}{\footnotesize{(+4.3)}}
& \textcolor{teal}{\footnotesize{(+2.0)}}
& \textcolor{teal}{\footnotesize{(+7.9)}}
& \textcolor{teal}{\footnotesize{(+1.0)}}
& \textcolor{red}{\footnotesize{(-0.7)}}
& \textcolor{teal}{\footnotesize{(+8.0)}}
& \textcolor{teal}{\footnotesize{(+0.7)}}
& \textcolor{teal}{\footnotesize{(+3.3)}} \\
\midrule
\addlinespace[0.35em]
\multicolumn{9}{@{}l}{\textbf{\textsc{Qwen-3B}}} \\
\addlinespace[0.15em]
Standalone VLM & \textbf{64.5} & 21.9 & 28.8 & 50.1 & 24.2 & 13.4 & 39.6 & 34.6 \\
Pretrained + Reasoner & 59.7 & 32.8 & 29.2 & 55.2 & \textbf{34.7} & 14.2 & 47.2 & 39.0 \\
\textbf{AC-RL (ours)} 
& 63.8 & \textbf{36.8} & \textbf{34.4} & \textbf{57.7} & 33.2 & \textbf{24.8} & \textbf{53.0} & \textbf{43.4} \\
\phantom{\textbf{AC-RL (ours)}} 
& \textcolor{teal}{\footnotesize{(+4.1)}} 
& \textcolor{teal}{\footnotesize{(+4.0)}} 
& \textcolor{teal}{\footnotesize{(+5.2)}} 
& \textcolor{teal}{\footnotesize{(+2.5)}} 
& \textcolor{red}{\footnotesize{(-1.5)}} 
& \textcolor{teal}{\footnotesize{(+10.6)}} 
& \textcolor{teal}{\footnotesize{(+5.8)}} 
& \textcolor{teal}{\footnotesize{(+4.4)}} \\
\bottomrule
\multicolumn{9}{l}{\footnotesize{$\dagger$ Result on \emph{testmini}/\emph{mini} subset.}}
\end{tabular}
}

\vspace{-0.5cm}
\end{table*}
\subsection{Ablations}
\label{subsec:decomposition}

To better understand the source of these improvements, we analyze the incremental value of each component in our approach using the Qwen2.5-VL-3B model (Table~\ref{tab:ablation_scaffold}). All configurations are evaluated using direct inference (no clarification allowed).

\begin{table*}[!h]
\centering
\caption{\textbf{Decomposition of performance gains on Qwen2.5-VL-3B} across multi-modal reasoning benchmarks: MathVista (MVista), MathVision (MVision), MathVerse (MVerse), MMMU, WeMath (WeM), DynaMath (DynaM), and LogicVista (LVista). All models are evaluated in a single-pass setting. The results show that AC-RL provides a significant performance boost beyond both architectural decoupling and Binary Rewards.}
\label{tab:ablation_scaffold}
\resizebox{\textwidth}{!}{
\begin{tabular}{@{}l|ccccccc|c@{}}
\toprule
\textbf{Training Method} & \textbf{MVista} & \textbf{MVision} & \textbf{MVerse} & \textbf{MMMU} & \textbf{WeM} & \textbf{DynaM} & \textbf{LVista} & \textbf{AVG} \\
\midrule
VLM-only (No Reasoner)       & \textbf{64.50} & 21.90 & 28.80 & 50.10 & 24.20   & 13.40 & 39.60 & 34.64 \\
Decoupled (No RL)            & 59.69 & 32.80 & 29.18 & 55.22 & \textbf{34.71} & 14.17 & 47.20 & 39.00 \\
Binary-Reward RL                & 62.60 & 34.30 & 31.09 & 55.44 & 33.45 & 17.56 & 47.42 & 40.27 \\
\textbf{AC-RL (Ours)}  & 63.80 & \textbf{36.84} & \textbf{34.39} & \textbf{57.70} & 33.22 & \textbf{24.75} & \textbf{53.02} & \textbf{43.39} \\
\bottomrule
\end{tabular}}
\end{table*}

The results illustrate a clear progression. First, decoupling the captioner from the reasoner (Pretrained + Reasoner) yields substantial gains, particularly on structurally complex tasks like \emph{MathVision} (21.9 $\to$ 32.8). Second, applying binary-reward RL provides further improvements across most benchmarks. However, our clarification-aware AC-RL delivers the most substantial gains over Binary-Reward RL: \emph{DynaMath} improves by +7.2 points (17.56 $\to$ 24.75), \emph{LogicVista} by +5.6 points, and \emph{MathVerse} by +3.3 points. 

The improvements on DynaMath are especially noteworthy. While Binary-Reward RL achieves modest gains over the pretrained baseline (+3.4 points), AC-RL delivers an additional +7.2 points, reaching 24.75\% accuracy on the most challenging problem variants.

\subsection{Analysis of Interface Behavior}
\label{subsec:mechanism}

To understand the mechanism underlying AC-RL's performance gains, we analyze how the training procedure modifies the captioner's behavior. Our hypothesis is that the clarification-aware reward teaches the model to \emph{front-load} reasoner-salient information into the initial caption, thereby reducing the need for clarification.

We first measure the \textbf{clarification attempt rate} using clarification-enabled evaluation (for measurement purposes only) and counting how frequently it requests clarification when processing captions from AC-RL-trained versus baseline models. Table~\ref{tab:clar_rate} shows that AC-RL dramatically reduces the frequency of clarification requests. On MathVision, the clarification rate drops from 40.69\% (Binary-Reward RL baseline) to 28.95\% (AC-RL). On MathVerse, the reduction is even more pronounced at 39\%. This confirms that AC-RL-trained captioners learn to preemptively include information that would otherwise trigger follow-up questions.

\begin{table}[!h]
\vspace{-0.4cm}
\centering
\caption{Clarification attempt rate during clarification-enabled evaluation. Lower rates indicate more comprehensive initial captions.}
\vspace{0.25cm}
\label{tab:clar_rate}
\begin{tabular}{@{}lccc@{}}
\toprule
\textbf{Dataset} & \textbf{Binary-Reward} & \textbf{AC-RL Model} & \textbf{Reduction} \\
\midrule
MathVision & 40.69\% & 28.95\% & \textbf{29\%} \\
MathVerse\_MINI & 49.57\% & 30.28\% & \textbf{39\%} \\
\bottomrule
\end{tabular}
\end{table}

Building on this evidence, we compute the \textbf{clarification gap}: the difference in accuracy between clarification-enabled evaluation and single-pass evaluation. A smaller gap indicates that the initial caption is more informationally self-sufficient. Table~\ref{tab:clar_gap} presents these results. For the baseline model, allowing clarification provides substantial accuracy gains: +2.89 points on MathVision and +4.06 points on MathVerse. In contrast, the AC-RL model shows minimal benefit from clarification and even a negative gap on MathVerse (-2.54), suggesting that its initial captions are so well-aligned that additional clarification can sometimes introduce noise. The relative improvement metric (fraction of previously incorrect answers that become correct with clarification) further confirms this pattern: AC-RL achieves 1.5\% relative improvement on MathVision versus 4.4\% for the baseline.

\begin{table}[!h]
\vspace{-0.25cm}
\centering
\caption{Performance gap between clarification-enabled and single-pass evaluation. Smaller gaps indicate greater self-sufficiency of initial captions. ``Rel'' denotes the fraction of previously incorrect answers corrected by clarification.}
\vspace{0.25cm}
\label{tab:clar_gap}
\setlength{\tabcolsep}{4pt}{
\resizebox{\textwidth}{!}{%
    \begin{tabular}{@{}lcccc@{}}
    \toprule
    \textbf{Dataset} & \textbf{Model} & \textbf{Clarification-Enabled} & \textbf{Single-Pass} & \textbf{Gap (Abs / Rel)} \\
    \midrule
    \multirow{2}{*}{MathVision}
    & AC-RL     & 37.66\% & 36.71\% & \textbf{+0.95 / 0.015} \\
    & Binary-Reward & 37.20\% & 34.31\% & +2.89 / 0.044 \\
    \midrule
    \multirow{2}{*}{MathVerse\_MINI}
    & AC-RL     & 34.26\% & 36.80\% & \textbf{-2.54 / -0.040} \\
    & Binary-Reward & 35.15\% & 31.09\% & +4.06 / 0.059 \\
    \bottomrule
    \end{tabular}
 }
}
\end{table}

Finally, to assess whether clarification requests are genuinely necessary, we measure accuracy under \textbf{denied clarification}: we identify instances where the model requested clarification during clarification-enabled evaluation, then examine the single-pass accuracy on this same subset (equivalent to denying the clarification request). The drop $\Delta_{\text{deny}} = \text{Acc}_{\text{clarification-enabled}} - \text{Acc}_{\text{denied}}$ quantifies how much the model relies on clarification when it requests it. Table~\ref{tab:deny} shows that while AC-RL reduces overall clarification frequency, its remaining requests are more selective. The AC-RL model exhibits a larger performance drop when clarification is denied ($\Delta_{\text{deny}} = 14.21$ on MathVision versus 11.43 for the baseline), despite making fewer requests overall (880 versus 1,237). This suggests that AC-RL learns to distinguish between recoverable and irrecoverable information gaps: it produces self-sufficient captions when possible, but when it does request clarification, these requests target instances where critical visual details cannot be inferred from context alone.

\begin{table}[h]
\centering
\caption{Accuracy impact of denying clarification on instances where it was requested. Larger drops indicate more critical clarification requests.}
\vspace{0.25cm}
\label{tab:deny}
\begin{tabular}{@{}lcccc@{}}
\toprule
\textbf{Dataset} & \textbf{Model} & \textbf{\# Requests} & \textbf{Acc$_{\text{single-pass}}$} & \textbf{Acc$_{\text{deny}}$ / $\Delta_{\text{deny}}$} \\
\midrule
\multirow{2}{*}{MathVision}
& AC-RL     & 880  & 36.71\% & 22.50\% \; / \; \textbf{14.21} \\
& Binary-Reward & 1237 & 34.31\% & 22.88\% \; / \; 11.43 \\
\midrule
\multirow{2}{*}{MathVerse\_MINI}
& AC-RL     & 276  & 36.80\% & 33.33\% \; / \; \textbf{3.47} \\
& Binary-Reward & 496  & 31.09\% & 28.43\% \; / \; 2.66 \\
\bottomrule
\end{tabular}
\end{table}

\subsection{Subject-Level Performance Analysis}
\label{subsec:subjectlevel}
\begin{figure}[!h]
\centering
\includegraphics[width=\textwidth]{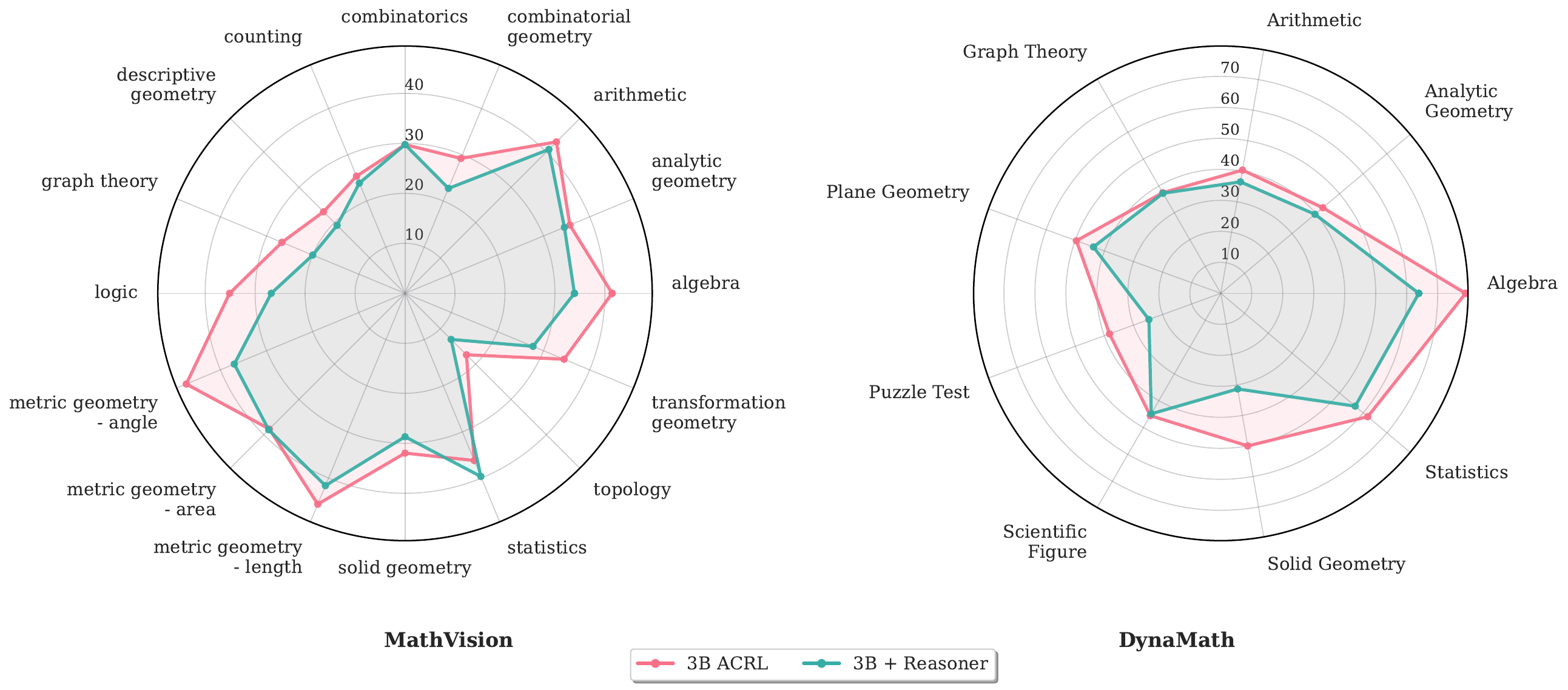}
\caption{Subject-level performance comparing AC-RL to the pretrained baseline using Qwen-3B + Reasoner. Left: MathVision subjects. Right: DynaMath categories. AC-RL shows targeted improvements in quantitatively-intensive domains like geometry and algebra.}
\label{fig:sla_analysis}
\end{figure}
To better understand the nature of these improvements, we conduct a fine-grained analysis of performance across different mathematical subjects and difficulty levels. This reveals whether the model is generically improving or learning to prioritize specific types of information relevant to the reasoner. We decompose the MathVision and DynaMath benchmarks by subject area and compute per-subject accuracy for both the AC-RL model and the pretrained baseline (Qwen-3B + Reasoner configuration). Additionally, we analyze DynaMath average performance stratified by education level (elementary, high school, undergraduate) to assess whether AC-RL's benefits vary with problem complexity. 

Figure~\ref{fig:sla_analysis} visualizes the per-subject performance comparison. The analysis reveals that AC-RL's gains are concentrated in subjects that depend heavily on precise quantitative and structural information. On MathVision, we observe the largest improvements in metric geometry for angles (+10.4 points), transformation geometry (+8.3 points), and algebra (+7.5 points). DynaMath shows even more pronounced gains in solid geometry (+18.7 points), algebra (+15.1 points), and puzzle tests (+13.5 points). These subjects arguably require extracting specific numerical values, spatial relationships, or structural patterns from images. In contrast, performance differences are minimal in subjects that rely more on general visual understanding or pattern recognition.

\begin{wrapfigure}{!hr}{0.4\textwidth}
\centering
\vspace{-0.45cm}
\includegraphics[width=0.35\textwidth]{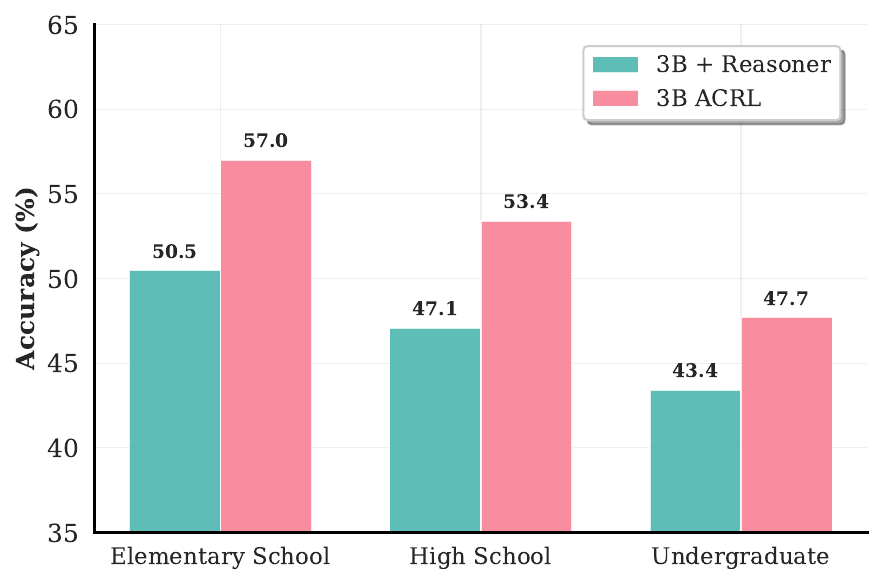}
\caption{DynaMath average accuracy across education levels. AC-RL consistently outperforms the baseline regardless of problem difficulty.}
\vspace{-1cm}
\label{fig:levels}
\end{wrapfigure}

Figure~\ref{fig:levels} shows that AC-RL maintains consistent improvements across all difficulty levels on DynaMath. The absolute gains remain relatively stable at 6.5, 6.3, and 4.3 points for elementary, high school, and undergraduate levels, respectively. While both models show expected degradation as problem complexity increases, AC-RL preserves its advantage by learning to extract critical visual details needed at each level. The slightly smaller gain at the undergraduate level may reflect inherent limits in what visual information alone can contribute to highly abstract problems.

This non-uniform improvement pattern indicates that AC-RL learns to extract and prioritize the specific types of information most valuable to the downstream reasoner. The selective nature of these improvements suggests that the clarification-aware training successfully identifies and addresses systematic information gaps in the original captioning policy, with the model discovering domain-specific extraction strategies through interaction rather than explicit supervision.

\section{Conclusion}
We presented Adaptive-Clarification Reinforcement Learning (AC-RL), a framework that learns vision-reasoner interfaces through interaction rather than supervision. By using clarification requests as implicit feedback and tiered rewards, AC-RL enables captioners to discover what information their paired reasoner requires without explicit annotation. Our experiments demonstrate consistent improvements across seven mathematical reasoning benchmarks, with particularly strong gains on quantitatively-intensive domains.
The success of AC-RL suggests that interface alignment between AI modules is a learnable property that can be optimized through reinforcement learning. This principle of using interaction patterns as learning signals extends beyond vision-language systems to any modular architecture where components coordinate through natural language. Future work could explore multi-turn clarification with decaying rewards, bidirectional adaptation where both modules co-evolve, and applications to other language-producing tools that interface with orchestrating LLMs.

\section*{Acknowledgments}
The authors thank the Netherlands Organization for Scientific Research (NWO) for their support (VICI grant VI.C.212.053).

This work used the Dutch national e-infrastructure with the support of the SURF Cooperative using grant no. EINF-15485.

\bibliographystyle{iclr2026_conference}
\bibliography{bibliography}

\begin{thebibliography}{43}
\providecommand{\natexlab}[1]{#1}
\providecommand{\url}[1]{\texttt{#1}}
\expandafter\ifx\csname urlstyle\endcsname\relax
  \providecommand{\doi}[1]{doi: #1}\else
  \providecommand{\doi}{doi: \begingroup \urlstyle{rm}\Url}\fi

\bibitem[Achiam et~al.(2023)Achiam, Adler, Agarwal, Ahmad, Akkaya, Aleman, Almeida, Altenschmidt, Altman, Anadkat, et~al.]{gpt4}
Josh Achiam, Steven Adler, Sandhini Agarwal, Lama Ahmad, Ilge Akkaya, Florencia~Leoni Aleman, Diogo Almeida, Janko Altenschmidt, Sam Altman, Shyamal Anadkat, et~al.
\newblock Gpt-4 technical report.
\newblock \emph{arXiv preprint arXiv:2303.08774}, 2023.

\bibitem[Bai et~al.(2025)Bai, Chen, Liu, Wang, Ge, Song, Dang, Wang, Wang, Tang, Zhong, Zhu, Yang, Li, Wan, Wang, Ding, Fu, Xu, Ye, Zhang, Xie, Cheng, Zhang, Yang, Xu, and Lin]{qwen25vl}
Shuai Bai, Keqin Chen, Xuejing Liu, Jialin Wang, Wenbin Ge, Sibo Song, Kai Dang, Peng Wang, Shijie Wang, Jun Tang, Humen Zhong, Yuanzhi Zhu, Mingkun Yang, Zhaohai Li, Jianqiang Wan, Pengfei Wang, Wei Ding, Zheren Fu, Yiheng Xu, Jiabo Ye, Xi~Zhang, Tianbao Xie, Zesen Cheng, Hang Zhang, Zhibo Yang, Haiyang Xu, and Junyang Lin.
\newblock Qwen2.5-vl technical report, 2025.
\newblock URL \url{https://arxiv.org/abs/2502.13923}.

\bibitem[Chen et~al.(2023)Chen, Li, Shen, Yang, Li, Keutzer, Darrell, and Liu]{cola}
Liangyu Chen, Bo~Li, Sheng Shen, Jingkang Yang, Chunyuan Li, Kurt Keutzer, Trevor Darrell, and Ziwei Liu.
\newblock Large language models are visual reasoning coordinators.
\newblock In \emph{Thirty-seventh Conference on Neural Information Processing Systems}, 2023.
\newblock URL \url{https://openreview.net/forum?id=1q0feiJ2i4}.

\bibitem[Chen et~al.(2024)Chen, Zhou, Shen, Hong, Sun, Gutfreund, and Gan]{vctp}
Zhenfang Chen, Qinhong Zhou, Yikang Shen, Yining Hong, Zhiqing Sun, Dan Gutfreund, and Chuang Gan.
\newblock Visual chain-of-thought prompting for knowledge-based visual reasoning.
\newblock \emph{Proceedings of the AAAI Conference on Artificial Intelligence}, 38\penalty0 (2):\penalty0 1254--1262, Mar. 2024.
\newblock \doi{10.1609/aaai.v38i2.27888}.
\newblock URL \url{https://ojs.aaai.org/index.php/AAAI/article/view/27888}.

\bibitem[Comanici et~al.(2025)Comanici, Bieber, Schaekermann, Pasupat, Sachdeva, Dhillon, Blistein, Ram, Zhang, Rosen, et~al.]{gemini2.5}
Gheorghe Comanici, Eric Bieber, Mike Schaekermann, Ice Pasupat, Noveen Sachdeva, Inderjit Dhillon, Marcel Blistein, Ori Ram, Dan Zhang, Evan Rosen, et~al.
\newblock Gemini 2.5: Pushing the frontier with advanced reasoning, multimodality, long context, and next generation agentic capabilities.
\newblock \emph{arXiv preprint arXiv:2507.06261}, 2025.

\bibitem[Deng et~al.(2025)Deng, Bansal, Yin, Peng, Wang, and Chang]{openvlthinker}
Yihe Deng, Hritik Bansal, Fan Yin, Nanyun Peng, Wei Wang, and Kai-Wei Chang.
\newblock Openvlthinker: Complex vision-language reasoning via iterative sft-rl cycles, 2025.
\newblock URL \url{https://arxiv.org/abs/2503.17352}.

\bibitem[Du et~al.(2023)Du, Li, Tang, Zhao, and Wen]{lamoc}
Yifan Du, Junyi Li, Tianyi Tang, Wayne~Xin Zhao, and Ji-Rong Wen.
\newblock Zero-shot visual question answering with language model feedback.
\newblock \emph{arXiv preprint arXiv:2305.17006}, 2023.

\bibitem[Duan et~al.(2024)Duan, Yang, Qiao, Fang, Chen, Liu, Dong, Zang, Zhang, Wang, et~al.]{vlmevalkit}
Haodong Duan, Junming Yang, Yuxuan Qiao, Xinyu Fang, Lin Chen, Yuan Liu, Xiaoyi Dong, Yuhang Zang, Pan Zhang, Jiaqi Wang, et~al.
\newblock Vlmevalkit: An open-source toolkit for evaluating large multi-modality models.
\newblock In \emph{Proceedings of the 32nd ACM International Conference on Multimedia}, pp.\  11198--11201, 2024.

\bibitem[Dzabraev et~al.(2024)Dzabraev, Kunitsyn, and Ivaniuta]{vlrm}
Maksim Dzabraev, Alexander Kunitsyn, and Andrei Ivaniuta.
\newblock Vlrm: Vision-language models act as reward models for image captioning.
\newblock \emph{arXiv preprint arXiv:2404.01911}, 2024.

\bibitem[Gou et~al.(2025)Gou, Chen, Liu, Hong, Jin, Li, Kwok, and Zhang]{racro}
Yunhao Gou, Kai Chen, Zhili Liu, Lanqing Hong, Xin Jin, Zhenguo Li, James~T. Kwok, and Yu~Zhang.
\newblock Perceptual decoupling for scalable multi-modal reasoning via reward-optimized captioning, 2025.
\newblock URL \url{https://arxiv.org/abs/2506.04559}.

\bibitem[Guo et~al.(2025{\natexlab{a}})Guo, Yang, Zhang, Song, Wang, Zhu, Xu, Zhang, Ma, Bi, et~al.]{deepseek_r1}
Daya Guo, Dejian Yang, Haowei Zhang, Junxiao Song, Peiyi Wang, Qihao Zhu, Runxin Xu, Ruoyu Zhang, Shirong Ma, Xiao Bi, et~al.
\newblock Deepseek-r1 incentivizes reasoning in llms through reinforcement learning.
\newblock \emph{Nature}, 645\penalty0 (8081):\penalty0 633--638, 2025{\natexlab{a}}.

\bibitem[Guo et~al.(2025{\natexlab{b}})Guo, Wu, Zhu, Leng, Shi, Chen, Fan, Wang, Jiang, Wang, Chen, Huang, Lei, Yuan, Luo, Liu, Ye, Qian, Yan, Zhao, Peng, Li, Yuan, Wu, Cheng, Liu, Wang, Zeng, Liu, Qin, Ding, Xiao, Zhang, Zhang, Xiong, Peng, Chen, Li, Hu, Lin, Hu, Zhang, Wu, Li, Liu, Ling, Qin, Wang, He, Zhang, Yi, Liao, Huang, Zhang, Deng, Deng, Lin, Yuan, Li, Gou, Lou, Wei, Liu, Li, Zhu, Zhong, Li, Zhang, Wu, Li, Xiao, Lin, Yang, Wang, Ji, Hao, Shen, Li, Li, Wu, Zhu, Jiao, Feng, Chen, Duan, Liu, Zeng, Tang, Sun, Chen, Long, Feng, Zhan, Fang, Lu, Hua, Liu, Shen, Zhang, Shen, Wang, Pan, Zhang, Li, Li, Li, Shi, Han, Xiang, Chen, Chen, Li, Yan, Chi, Liu, Du, Wang, Pan, Chen, Chen, Wu, Yuan, Shuai, Tao, Zheng, Zhang, Zhang, Wang, Yang, Zhao, Xu, Liang, Yan, Zhong, Cao, Wu, Liu, Chang, Cai, Ao, Yang, Zhang, Zhong, Jia, Weng, Yu, Huang, Zhu, Yang, Wang, Long, Yin, Li, Zhu, Jia, Zhang, Liu, Zhang, Yang, Luo, Chen, Zhong, Xiao, Li, Wu, Wen, Du, Zhang, Ye, Wu, Liu, Yue, Zhou, Yuan, Xu, Yang, Zhang, Fang, Li, Ren,
  Xiong, Hong, Wang, Sun, Wang, Cai, Zha, An, Zhao, Xu, Chen, Wu, Zheng, Wang, Huang, Zhu, and Song]{seed15vl}
Dong Guo, Faming Wu, Feida Zhu, Fuxing Leng, Guang Shi, Haobin Chen, Haoqi Fan, Jian Wang, Jianyu Jiang, Jiawei Wang, Jingji Chen, Jingjia Huang, Kang Lei, Liping Yuan, Lishu Luo, Pengfei Liu, Qinghao Ye, Rui Qian, Shen Yan, Shixiong Zhao, Shuai Peng, Shuangye Li, Sihang Yuan, Sijin Wu, Tianheng Cheng, Weiwei Liu, Wenqian Wang, Xianhan Zeng, Xiao Liu, Xiaobo Qin, Xiaohan Ding, Xiaojun Xiao, Xiaoying Zhang, Xuanwei Zhang, Xuehan Xiong, Yanghua Peng, Yangrui Chen, Yanwei Li, Yanxu Hu, Yi~Lin, Yiyuan Hu, Yiyuan Zhang, Youbin Wu, Yu~Li, Yudong Liu, Yue Ling, Yujia Qin, Zanbo Wang, Zhiwu He, Aoxue Zhang, Bairen Yi, Bencheng Liao, Can Huang, Can Zhang, Chaorui Deng, Chaoyi Deng, Cheng Lin, Cheng Yuan, Chenggang Li, Chenhui Gou, Chenwei Lou, Chengzhi Wei, Chundian Liu, Chunyuan Li, Deyao Zhu, Donghong Zhong, Feng Li, Feng Zhang, Gang Wu, Guodong Li, Guohong Xiao, Haibin Lin, Haihua Yang, Haoming Wang, Heng Ji, Hongxiang Hao, Hui Shen, Huixia Li, Jiahao Li, Jialong Wu, Jianhua Zhu, Jianpeng Jiao, Jiashi Feng, Jiaze
  Chen, Jianhui Duan, Jihao Liu, Jin Zeng, Jingqun Tang, Jingyu Sun, Joya Chen, Jun Long, Junda Feng, Junfeng Zhan, Junjie Fang, Junting Lu, Kai Hua, Kai Liu, Kai Shen, Kaiyuan Zhang, Ke~Shen, Ke~Wang, Keyu Pan, Kun Zhang, Kunchang Li, Lanxin Li, Lei Li, Lei Shi, Li~Han, Liang Xiang, Liangqiang Chen, Lin Chen, Lin Li, Lin Yan, Liying Chi, Longxiang Liu, Mengfei Du, Mingxuan Wang, Ningxin Pan, Peibin Chen, Pengfei Chen, Pengfei Wu, Qingqing Yuan, Qingyao Shuai, Qiuyan Tao, Renjie Zheng, Renrui Zhang, Ru~Zhang, Rui Wang, Rui Yang, Rui Zhao, Shaoqiang Xu, Shihao Liang, Shipeng Yan, Shu Zhong, Shuaishuai Cao, Shuangzhi Wu, Shufan Liu, Shuhan Chang, Songhua Cai, Tenglong Ao, Tianhao Yang, Tingting Zhang, Wanjun Zhong, Wei Jia, Wei Weng, Weihao Yu, Wenhao Huang, Wenjia Zhu, Wenli Yang, Wenzhi Wang, Xiang Long, XiangRui Yin, Xiao Li, Xiaolei Zhu, Xiaoying Jia, Xijin Zhang, Xin Liu, Xinchen Zhang, Xinyu Yang, Xiongcai Luo, Xiuli Chen, Xuantong Zhong, Xuefeng Xiao, Xujing Li, Yan Wu, Yawei Wen, Yifan Du, Yihao Zhang,
  Yining Ye, Yonghui Wu, Yu~Liu, Yu~Yue, Yufeng Zhou, Yufeng Yuan, Yuhang Xu, Yuhong Yang, Yun Zhang, Yunhao Fang, Yuntao Li, Yurui Ren, Yuwen Xiong, Zehua Hong, Zehua Wang, Zewei Sun, Zeyu Wang, Zhao Cai, Zhaoyue Zha, Zhecheng An, Zhehui Zhao, Zhengzhuo Xu, Zhipeng Chen, Zhiyong Wu, Zhuofan Zheng, Zihao Wang, Zilong Huang, Ziyu Zhu, and Zuquan Song.
\newblock Seed1.5-vl technical report, 2025{\natexlab{b}}.
\newblock URL \url{https://arxiv.org/abs/2505.07062}.

\bibitem[Guo et~al.(2025{\natexlab{c}})Guo, Liu, Wang, Ji, Bai, Zhang, and Zuo]{integrating_visual_interpretation}
Zixian Guo, Ming Liu, Qilong Wang, Zhilong Ji, Jinfeng Bai, Lei Zhang, and Wangmeng Zuo.
\newblock Integrating visual interpretation and linguistic reasoning for math problem solving, 2025{\natexlab{c}}.
\newblock URL \url{https://arxiv.org/abs/2505.17609}.

\bibitem[Gupta \& Kembhavi(2023)Gupta and Kembhavi]{visprog}
Tanmay Gupta and Aniruddha Kembhavi.
\newblock Visual programming: Compositional visual reasoning without training.
\newblock In \emph{Proceedings of the IEEE/CVF conference on computer vision and pattern recognition}, pp.\  14953--14962, 2023.

\bibitem[Hu et~al.()Hu, Wallis, Allen-Zhu, Li, Wang, Wang, Chen, et~al.]{lora}
Edward~J Hu, Phillip Wallis, Zeyuan Allen-Zhu, Yuanzhi Li, Shean Wang, Lu~Wang, Weizhu Chen, et~al.
\newblock Lora: Low-rank adaptation of large language models.
\newblock In \emph{International Conference on Learning Representations}.

\bibitem[Huang et~al.(2025)Huang, Jia, Zhai, Cao, Ye, Zhao, Xu, Hu, and Lin]{vision-r1}
Wenxuan Huang, Bohan Jia, Zijie Zhai, Shaosheng Cao, Zheyu Ye, Fei Zhao, Zhe Xu, Yao Hu, and Shaohui Lin.
\newblock Vision-r1: Incentivizing reasoning capability in multimodal large language models.
\newblock \emph{arXiv preprint arXiv:2503.06749}, 2025.

\bibitem[Jaech et~al.(2024)Jaech, Kalai, Lerer, Richardson, El-Kishky, Low, Helyar, Madry, Beutel, Carney, et~al.]{openai_o1}
Aaron Jaech, Adam Kalai, Adam Lerer, Adam Richardson, Ahmed El-Kishky, Aiden Low, Alec Helyar, Aleksander Madry, Alex Beutel, Alex Carney, et~al.
\newblock Openai o1 system card.
\newblock \emph{arXiv preprint arXiv:2412.16720}, 2024.

\bibitem[Leng et~al.(2025)Leng, Wang, Li, Zhang, Hu, Zhang, Jiang, Zhang, Li, Bing, Zhao, Lu, Rong, Sun, and Lu]{MMR1}
Sicong Leng, Jing Wang, Jiaxi Li, Hao Zhang, Zhiqiang Hu, Boqiang Zhang, Yuming Jiang, Hang Zhang, Xin Li, Lidong Bing, Deli Zhao, Wei Lu, Yu~Rong, Aixin Sun, and Shijian Lu.
\newblock Mmr1: Enhancing multimodal reasoning with variance-aware sampling and open resources, 2025.
\newblock URL \url{https://arxiv.org/abs/2509.21268}.

\bibitem[Lu et~al.(2024{\natexlab{a}})Lu, Bansal, Xia, Liu, Li, Hajishirzi, Cheng, Chang, Galley, and Gao]{mathvista}
Pan Lu, Hritik Bansal, Tony Xia, Jiacheng Liu, Chunyuan Li, Hannaneh Hajishirzi, Hao Cheng, Kai-Wei Chang, Michel Galley, and Jianfeng Gao.
\newblock Mathvista: Evaluating mathematical reasoning of foundation models in visual contexts.
\newblock In \emph{International Conference on Learning Representations (ICLR)}, 2024{\natexlab{a}}.

\bibitem[Lu et~al.(2024{\natexlab{b}})Lu, Li, Chen, Xu, Luo, Zhang, and Ye]{ovis}
Shiyin Lu, Yang Li, Qing-Guo Chen, Zhao Xu, Weihua Luo, Kaifu Zhang, and Han-Jia Ye.
\newblock Ovis: Structural embedding alignment for multimodal large language model.
\newblock \emph{arXiv:2405.20797}, 2024{\natexlab{b}}.

\bibitem[Lu et~al.(2025)Lu, Yuan, Li, Zhao, Qin, Li, Zhuo, Wen, Liu, Cao, Yan, Li, Peng, Zhang, Shi, Chen, Chen, Bai, Gao, and Zhang]{omnicaptioner}
Yiting Lu, Jiakang Yuan, Zhen Li, Shitian Zhao, Qi~Qin, Xinyue Li, Le~Zhuo, Licheng Wen, Dongyang Liu, Yuewen Cao, Xiangchao Yan, Xin Li, Tianshuo Peng, Shufei Zhang, Botian Shi, Tao Chen, Zhibo Chen, Lei Bai, Peng Gao, and Bo~Zhang.
\newblock Omnicaptioner: One captioner to rule them all, 2025.
\newblock URL \url{https://arxiv.org/abs/2504.07089}.

\bibitem[Qiao et~al.(2024)Qiao, Tan, Dong, Wu, Sun, Song, GongQue, Lei, Wei, Zhang, et~al.]{wemath}
Runqi Qiao, Qiuna Tan, Guanting Dong, Minhui Wu, Chong Sun, Xiaoshuai Song, Zhuoma GongQue, Shanglin Lei, Zhe Wei, Miaoxuan Zhang, et~al.
\newblock We-math: Does your large multimodal model achieve human-like mathematical reasoning?
\newblock \emph{arXiv preprint arXiv:2407.01284}, 2024.

\bibitem[Shao et~al.(2024)Shao, Wang, Zhu, Xu, Song, Bi, Zhang, Zhang, Li, Wu, et~al.]{deepseek_math}
Zhihong Shao, Peiyi Wang, Qihao Zhu, Runxin Xu, Junxiao Song, Xiao Bi, Haowei Zhang, Mingchuan Zhang, YK~Li, Yang Wu, et~al.
\newblock Deepseekmath: Pushing the limits of mathematical reasoning in open language models.
\newblock \emph{arXiv preprint arXiv:2402.03300}, 2024.

\bibitem[Singh et~al.(2024)Singh, Gupta, Garg, Kumar, and Agrawal]{beyond_captioning}
Ayush Singh, Mansi Gupta, Shivank Garg, Abhinav Kumar, and Vansh Agrawal.
\newblock Beyond captioning: Task-specific prompting for improved vlm performance in mathematical reasoning, 2024.
\newblock URL \url{https://arxiv.org/abs/2410.05928}.

\bibitem[Team et~al.(2023)Team, Anil, Borgeaud, Alayrac, Yu, Soricut, Schalkwyk, Dai, Hauth, Millican, et~al.]{gemini}
Gemini Team, Rohan Anil, Sebastian Borgeaud, Jean-Baptiste Alayrac, Jiahui Yu, Radu Soricut, Johan Schalkwyk, Andrew~M Dai, Anja Hauth, Katie Millican, et~al.
\newblock Gemini: a family of highly capable multimodal models.
\newblock \emph{arXiv preprint arXiv:2312.11805}, 2023.

\bibitem[Team(2024{\natexlab{a}})]{evalscope_2024}
ModelScope Team.
\newblock {EvalScope}: Evaluation framework for large models, 2024{\natexlab{a}}.
\newblock URL \url{https://github.com/modelscope/evalscope}.

\bibitem[Team(2024{\natexlab{b}})]{qvq}
Qwen Team.
\newblock Qvq: To see the world with wisdom, December 2024{\natexlab{b}}.
\newblock URL \url{https://qwenlm.github.io/blog/qvq-72b-preview/}.

\bibitem[von Werra et~al.(2020)von Werra, Belkada, Tunstall, Beeching, Thrush, Lambert, Huang, Rasul, and Gallouédec]{TRL}
Leandro von Werra, Younes Belkada, Lewis Tunstall, Edward Beeching, Tristan Thrush, Nathan Lambert, Shengyi Huang, Kashif Rasul, and Quentin Gallouédec.
\newblock Trl: Transformer reinforcement learning.
\newblock \url{https://github.com/huggingface/trl}, 2020.

\bibitem[Wang et~al.(2025{\natexlab{a}})Wang, Qu, Huang, Chu, Lin, and Chen]{vl_rethinker}
Haozhe Wang, Chao Qu, Zuming Huang, Wei Chu, Fangzhen Lin, and Wenhu Chen.
\newblock Vl-rethinker: Incentivizing self-reflection of vision-language models with reinforcement learning.
\newblock \emph{arXiv preprint arXiv:2504.08837}, 2025{\natexlab{a}}.

\bibitem[Wang et~al.(2024)Wang, Pan, Shi, Lu, Ren, Zhou, Zhan, and Li]{mathvision}
Ke~Wang, Junting Pan, Weikang Shi, Zimu Lu, Houxing Ren, Aojun Zhou, Mingjie Zhan, and Hongsheng Li.
\newblock Measuring multimodal mathematical reasoning with math-vision dataset.
\newblock In \emph{The Thirty-eight Conference on Neural Information Processing Systems Datasets and Benchmarks Track}, 2024.
\newblock URL \url{https://openreview.net/forum?id=QWTCcxMpPA}.

\bibitem[Wang et~al.(2025{\natexlab{b}})Wang, Wei, Peng, Wang, Qiu, Shen, Xie, Pei, Zhang, Hao, et~al.]{skywork_r1v2}
Peiyu Wang, Yichen Wei, Yi~Peng, Xiaokun Wang, Weijie Qiu, Wei Shen, Tianyidan Xie, Jiangbo Pei, Jianhao Zhang, Yunzhuo Hao, et~al.
\newblock Skywork r1v2: Multimodal hybrid reinforcement learning for reasoning.
\newblock \emph{arXiv preprint arXiv:2504.16656}, 2025{\natexlab{b}}.

\bibitem[Wu et~al.(2024)Wu, Bansal, Zhang, Wu, Li, Zhu, Jiang, Zhang, Zhang, Liu, Awadallah, White, Burger, and Wang]{wu2024autogen}
Qingyun Wu, Gagan Bansal, Jieyu Zhang, Yiran Wu, Beibin Li, Erkang Zhu, Li~Jiang, Xiaoyun Zhang, Shaokun Zhang, Jiale Liu, Ahmed~Hassan Awadallah, Ryen~W White, Doug Burger, and Chi Wang.
\newblock Autogen: Enabling next-gen {LLM} applications via multi-agent conversations.
\newblock In \emph{First Conference on Language Modeling}, 2024.
\newblock URL \url{https://openreview.net/forum?id=BAakY1hNKS}.

\bibitem[Xiao et~al.(2025)Xiao, Zhang, and Cao]{bnpo}
Changyi Xiao, Mengdi Zhang, and Yixin Cao.
\newblock Bnpo: Beta normalization policy optimization, 2025.
\newblock URL \url{https://arxiv.org/abs/2506.02864}.

\bibitem[Xiao et~al.(2024)Xiao, Sun, Liu, and Wang]{logicvista}
Yijia Xiao, Edward Sun, Tianyu Liu, and Wei Wang.
\newblock Logicvista: Multimodal llm logical reasoning benchmark in visual contexts.
\newblock \emph{arXiv preprint arXiv:2407.04973}, 2024.

\bibitem[Yang et~al.(2025)Yang, He, Pan, Jiang, Deng, Yang, Lu, Yin, Rao, Zhu, et~al.]{r1-onevision}
Yi~Yang, Xiaoxuan He, Hongkun Pan, Xiyan Jiang, Yan Deng, Xingtao Yang, Haoyu Lu, Dacheng Yin, Fengyun Rao, Minfeng Zhu, et~al.
\newblock R1-onevision: Advancing generalized multimodal reasoning through cross-modal formalization.
\newblock \emph{arXiv preprint arXiv:2503.10615}, 2025.

\bibitem[Yao et~al.(2024)Yao, Huang, Wu, Zhang, Wang, Liu, Wang, Song, Feng, Shen, and Tao]{mulberry}
Huanjin Yao, Jiaxing Huang, Wenhao Wu, Jingyi Zhang, Yibo Wang, Shunyu Liu, Yingjie Wang, Yuxin Song, Haocheng Feng, Li~Shen, and Dacheng Tao.
\newblock Mulberry: Empowering mllm with o1-like reasoning and reflection via collective monte carlo tree search.
\newblock \emph{CoRR}, abs/2412.18319, 2024.
\newblock URL \url{https://doi.org/10.48550/arXiv.2412.18319}.

\bibitem[Yue et~al.(2024)Yue, Ni, Zhang, Zheng, Liu, Zhang, Stevens, Jiang, Ren, Sun, Wei, Yu, Yuan, Sun, Yin, Zheng, Yang, Liu, Huang, Sun, Su, and Chen]{mmmu}
Xiang Yue, Yuansheng Ni, Kai Zhang, Tianyu Zheng, Ruoqi Liu, Ge~Zhang, Samuel Stevens, Dongfu Jiang, Weiming Ren, Yuxuan Sun, Cong Wei, Botao Yu, Ruibin Yuan, Renliang Sun, Ming Yin, Boyuan Zheng, Zhenzhu Yang, Yibo Liu, Wenhao Huang, Huan Sun, Yu~Su, and Wenhu Chen.
\newblock Mmmu: A massive multi-discipline multimodal understanding and reasoning benchmark for expert agi.
\newblock In \emph{Proceedings of CVPR}, 2024.

\bibitem[Zhang et~al.(2025)Zhang, Lei, Li, Wang, Liu, Yang, Li, Wang, Yang, Wu, et~al.]{critic-v}
Di~Zhang, Jingdi Lei, Junxian Li, Xunzhi Wang, Yujie Liu, Zonglin Yang, Jiatong Li, Weida Wang, Suorong Yang, Jianbo Wu, et~al.
\newblock Critic-v: Vlm critics help catch vlm errors in multimodal reasoning.
\newblock In \emph{Proceedings of the Computer Vision and Pattern Recognition Conference}, pp.\  9050--9061, 2025.

\bibitem[Zhang et~al.(2024)Zhang, Jiang, Zhang, Lin, Guo, Qiu, Zhou, Lu, Chang, Qiao, et~al.]{mathverse}
Renrui Zhang, Dongzhi Jiang, Yichi Zhang, Haokun Lin, Ziyu Guo, Pengshuo Qiu, Aojun Zhou, Pan Lu, Kai-Wei Chang, Yu~Qiao, et~al.
\newblock Mathverse: Does your multi-modal llm truly see the diagrams in visual math problems?
\newblock In \emph{European Conference on Computer Vision}, pp.\  169--186. Springer, 2024.

\bibitem[Zhou et~al.(2024{\natexlab{a}})Zhou, Lee, Misu, and Wang]{vicor}
Kaiwen Zhou, Kwonjoon Lee, Teruhisa Misu, and Xin Wang.
\newblock {V}i{C}or: Bridging visual understanding and commonsense reasoning with large language models.
\newblock In Lun-Wei Ku, Andre Martins, and Vivek Srikumar (eds.), \emph{Findings of the Association for Computational Linguistics: ACL 2024}, pp.\  10783--10795, Bangkok, Thailand, August 2024{\natexlab{a}}. Association for Computational Linguistics.
\newblock \doi{10.18653/v1/2024.findings-acl.640}.
\newblock URL \url{https://aclanthology.org/2024.findings-acl.640/}.

\bibitem[Zhou et~al.(2024{\natexlab{b}})Zhou, Zhu, Antognini, Kim, and Zhang]{paraphrase_solve}
Yue Zhou, Yada Zhu, Diego Antognini, Yoon Kim, and Yang Zhang.
\newblock Paraphrase and solve: Exploring and exploiting the impact of surface form on mathematical reasoning in large language models.
\newblock In Kevin Duh, Helena Gomez, and Steven Bethard (eds.), \emph{Proceedings of the 2024 Conference of the North American Chapter of the Association for Computational Linguistics: Human Language Technologies (Volume 1: Long Papers)}, pp.\  2793--2804, Mexico City, Mexico, June 2024{\natexlab{b}}. Association for Computational Linguistics.
\newblock \doi{10.18653/v1/2024.naacl-long.153}.
\newblock URL \url{https://aclanthology.org/2024.naacl-long.153/}.

\bibitem[Zhu et~al.(2025)Zhu, Wang, Chen, Liu, Ye, Gu, Tian, Duan, Su, Shao, et~al.]{internvl3}
Jinguo Zhu, Weiyun Wang, Zhe Chen, Zhaoyang Liu, Shenglong Ye, Lixin Gu, Hao Tian, Yuchen Duan, Weijie Su, Jie Shao, et~al.
\newblock Internvl3: Exploring advanced training and test-time recipes for open-source multimodal models.
\newblock \emph{arXiv preprint arXiv:2504.10479}, 2025.

\bibitem[Zou et~al.(2025)Zou, Guo, Yang, Zhang, Hu, and Zhang]{dynamath}
Chengke Zou, Xingang Guo, Rui Yang, Junyu Zhang, Bin Hu, and Huan Zhang.
\newblock Dynamath: A dynamic visual benchmark for evaluating mathematical reasoning robustness of vision language models.
\newblock In \emph{The Thirteenth International Conference on Learning Representations}, 2025.
\newblock URL \url{https://openreview.net/forum?id=VOAMTA8jKu}.

\end{thebibliography}

\newpage
\appendix

\section{AC-RL Algorithm}
\label{app:algorithm}

We provide a formal specification of the Adaptive-Clarification Reinforcement Learning (AC-RL) algorithm. The algorithm operates in an episodic setting where each episode consists of a visual reasoning problem $(I, Q)$ sampled from the dataset $\mathcal{D}$.

\subsection{Formal Problem Setup}

Let $\mathcal{M} = (\mathcal{S}, \mathcal{A}, P, R, \gamma)$ denote the Markov Decision Process where:
\begin{itemize}
    \item $\mathcal{S} = \mathcal{I} \times \mathcal{Q} \times \mathcal{H}$ is the state space,
    \item $\mathscr{A} = \mathcal{C}$ is the action space (caption generation),
    \item $P: \mathcal{S} \times \mathcal{A} \rightarrow \Delta(\mathcal{S})$ is the transition kernel,
    \item $R: \mathcal{T} \rightarrow [0, 1]$ is the reward function defined on trajectories,
    \item $\gamma = 1$ (undiscounted episodic setting).
\end{itemize}

A single episode proceeds as follows. At $t=0$, the vision policy emits the initial caption $c_0 \sim \pi_\theta(\cdot \mid s_0)$ with $s_0=(I,Q,\emptyset)$. The reasoner stochastically decides whether to request clarification and, if so, which question to ask; we denote this by a $\theta$-independent kernel $q_1 \sim p(\cdot \mid s_0, c_0)$. When $q_1 \neq \emptyset$, the clarification caption is produced by a \emph{frozen} checkpoint $\pi_{\text{ref}}$:
\[
c_1 \sim \pi_{\text{ref}}(\cdot \mid s_1), \qquad s_1=(I,Q,\{c_0,q_1\}),
\]
and the reasoner produces a final answer according to a $\theta$-independent kernel $A \sim p(\cdot \mid Q, c_0, (q_1,c_1))$. When $q_1=\emptyset$, the reasoner answers from $(Q,c_0)$ directly. The next state appends the sampled variables to the dialogue history. Thus $P$ composes the reasoner’s stochastic behavior and the frozen clarification-caption policy $\pi_{\text{ref}}$; conditioned on the agent’s action $c_0$, these post-action mechanisms are \emph{$\theta$-independent} by construction. The episode terminates after the answer $A$ is produced, and the reward is assigned as in the main text.

\paragraph{Clarification captioning is frozen.}
In all experiments, the clarification caption $c_1$ is generated by a \emph{frozen} checkpoint $\pi_{\text{ref}}$ (typically the reference policy). Its distribution does not change during training. Consequently, no gradients flow through $\pi_{\text{ref}}$ or through the reasoner $\mathcal{R}$; only the log-probabilities of the initial caption tokens $c_0$ contribute to the policy update.

\subsection{Policy Update}

The policy is updated using the clipped surrogate objective with a fixed KL reference:
\begin{equation}
    \mathcal{L}_{\text{clip}}(\theta) = -\mathbb{E}_{(s_t, a_t)} \left[ \min \left( r_t(\theta) A_t, \operatorname{clip}(r_t(\theta), 1 - \epsilon, 1 + \epsilon) A_t \right) \right] + \beta \, D_{KL}(\pi_\theta \,\|\, \pi_{\text{ref-KL}}),
\end{equation}
where $r_t(\theta) = \pi_\theta(a_t \mid s_t) / \pi_{\theta_{\text{old}}}(a_t \mid s_t)$, and $A_t$ is the advantage computed via BNPO~\citep{bnpo}. The gradient is computed solely on the initial captioning segments $c_0$; the clarification responses $c_1$ are emitted by the frozen $\pi_{\text{ref}}$ and are thus $\theta$-independent.

\newpage

\subsection{Training Algorithm}

\begin{algorithm}[!h]
\caption{Adaptive-Clarification Reinforcement Learning (AC-RL)}
\label{alg:acrl}
\begin{algorithmic}[1]
\REQUIRE Dataset $\mathcal{D}$, vision model $\mathcal{V}_\theta$, reasoner $\mathcal{R}$, penalty $\alpha$, group size $M$, gradient steps $K$
\STATE Initialize policy $\pi_\theta \leftarrow \mathcal{V}_\theta$ with parameters $\theta$
\STATE Initialize \emph{frozen} clarification captioner $\pi_{\text{ref}}$ (checkpoint used only for $c_1$)
\STATE Initialize fixed KL reference $\pi_{\text{ref}} \leftarrow \pi_\theta$
\FOR{iteration $t = 1$ to $T$}
    \STATE Sample batch $\mathcal{B} = \{(I_j, Q_j)\}_{j=1}^B \sim \mathcal{D}$
    \FOR{each $(I_j, Q_j) \in \mathcal{B}$}
        \FOR{$i = 1$ to $M$}
            \STATE Generate initial caption: $c_0^{(i,j)} \sim \pi_\theta(\cdot \mid I_j, Q_j)$
            \STATE Sample reasoner’s clarification decision: $q_1^{(i,j)} \sim \mathcal{R}_{\text{clarify}}(\cdot \mid Q_j, c_0^{(i,j)})$ \hfill {\small (no gradients)}
            \IF{$q_1^{(i,j)} \neq \emptyset$}
                \STATE Generate clarification caption from \emph{frozen} checkpoint: $c_1^{(i,j)} \sim \pi_{\text{ref}}(\cdot \mid I_j, Q_j, (c_0^{(i,j)}, q_1^{(i,j)}))$ \hfill {\small (no gradients)}
                \STATE Get answer: $A^{(i,j)} \sim \mathcal{R}(\cdot \mid Q_j, c_0^{(i,j)}, (q_1^{(i,j)}, c_1^{(i,j)}))$ \hfill {\small (no gradients)}
                \STATE Set clarification flag: $C^{(i,j)} = 1$
            \ELSE
                \STATE Get answer: $A^{(i,j)} \sim \mathcal{R}(\cdot \mid Q_j, c_0^{(i,j)})$ \hfill {\small (no gradients)}
                \STATE Set clarification flag: $C^{(i,j)} = 0$
            \ENDIF
            \STATE Compute reward: $R^{(i,j)} = \begin{cases}
                1.0 & \text{if correct}(A^{(i,j)}) \land C^{(i,j)} = 0 \\
                \alpha & \text{if correct}(A^{(i,j)}) \land C^{(i,j)} = 1 \\
                0 & \text{otherwise}
            \end{cases}$
        \ENDFOR
        \STATE Fit Beta parameters $(\alpha_\beta^{(j)}, \beta_\beta^{(j)})$ to $\{R^{(i,j)}\}_{i=1}^M$
        \STATE Compute BNPO advantages $\{A_{\text{BNPO}}^{(i,j)}\}_{i=1}^M$
    \ENDFOR
    \FOR{$k = 1$ to $K$}
        \STATE Update policy with clipping and KL penalty: $\theta \leftarrow \theta - \eta \nabla_\theta \mathcal{L}_{\text{clip}}(\theta;\pi_{\text{ref-KL}})$
    \ENDFOR
\ENDFOR
\RETURN Trained policy $\pi_\theta$
\end{algorithmic}
\end{algorithm}

\newpage
\section{Unbiasedness of the Three-Tier Reward}
\label{app:proof}

In this section, we provide a formal proof that our three-tier reward structure maintains the unbiasedness property of the REINFORCE policy gradient estimator, even when the reasoner exhibits stochasticity.

\begin{theorem}[Unbiasedness of the Three-Tier Reward with Stochastic Reasoner]
\label{thm:unbiased}
Let $\xi \sim p(\cdot \mid \tau)$ denote all post-action randomness after the policy chooses its actions (e.g., the reasoner’s sampling noise and, when clarification is used, the frozen clarification-caption sampling). Define the extended trajectory $\tilde{\tau} = (\tau, \xi)$ with joint density:
\begin{equation}
p_\theta(\tilde{\tau}) = p_\theta(\tau) \cdot p(\xi \mid \tau)
\end{equation}
where $p(\xi \mid \tau)$ is independent of $\theta$.

Let the tiered reward function be defined as:
\begin{equation}
R_{\text{tier}}(\tilde{\tau}) = \begin{cases}
1 & \text{if } \text{correct}(A(\tilde{\tau})) \land C(\tilde{\tau}) = 0 \\
\alpha & \text{if } \text{correct}(A(\tilde{\tau})) \land C(\tilde{\tau}) > 0 \\
0 & \text{otherwise}
\end{cases}
\end{equation}
where $\alpha \in (0, 1)$, and define the training objective:
\begin{equation}
J(\theta) = \mathbb{E}_{\tilde{\tau} \sim p_\theta}[R_{\text{tier}}(\tilde{\tau})].
\end{equation}
For any baseline $b_t(s_t)$ that does not depend on the action $a_t$, the REINFORCE estimator:
\begin{equation}
\hat{g}(\tilde{\tau}) = \sum_{t=0}^{T-1} \nabla_\theta \log \pi_\theta(a_t \mid s_t) \cdot (R_{\text{tier}}(\tilde{\tau}) - b_t(s_t))
\end{equation}
satisfies $\mathbb{E}_{\tilde{\tau} \sim p_\theta}[\hat{g}(\tilde{\tau})] = \nabla_\theta J(\theta)$, i.e., the policy gradient remains unbiased despite the $0/\alpha/1$ reward shaping and post-action stochasticity.
\end{theorem}

\begin{proof}
\textbf{Step 1: Setup.}
The extended trajectory $\tilde{\tau} = (\tau, \xi)$ includes both the policy-generated trajectory $\tau = (s_0, a_0, s_1, a_1, ..., s_T)$ and the post-action randomness $\xi$. The joint probability decomposes as:
\begin{equation}
p_\theta(\tilde{\tau}) = p(s_0) \prod_{t=0}^{T-1} \pi_\theta(a_t \mid s_t) \, P(s_{t+1} \mid s_t, a_t) \cdot p(\xi \mid \tau),
\end{equation}
where $p(s_0)$ is the initial state distribution, $P$ is the environment transition kernel, and $p(\xi \mid \tau)$ is the distribution over post-action randomness given the trajectory; by assumption, $p(\xi \mid \tau)$ is $\theta$-independent.

\textbf{Step 2: Policy Gradient Theorem.}
For $R_{\text{tier}}(\tilde{\tau})$,
\begin{align}
\nabla_\theta J(\theta) &= \nabla_\theta \int p_\theta(\tau) \, p(\xi \mid \tau) \, R_{\text{tier}}(\tilde{\tau}) \, d\xi \, d\tau \\
&= \int p_\theta(\tau) \, p(\xi \mid \tau) \, R_{\text{tier}}(\tilde{\tau}) \, \nabla_\theta \log p_\theta(\tau) \, d\xi \, d\tau \\
&= \int p_\theta(\tilde{\tau}) \, R_{\text{tier}}(\tilde{\tau}) \, \nabla_\theta \log p_\theta(\tau) \, d\tilde{\tau},
\end{align}
using that $\nabla_\theta \log p_\theta(\tilde{\tau}) = \nabla_\theta \log p_\theta(\tau) + \nabla_\theta \log p(\xi\mid\tau)$ and $\nabla_\theta \log p(\xi\mid\tau)=0$ by $\theta$-independence. Since $p(s_0)$ and $P$ are $\theta$-independent,
\begin{equation}
\nabla_\theta \log p_\theta(\tau) = \sum_{t=0}^{T-1} \nabla_\theta \log \pi_\theta(a_t \mid s_t),
\end{equation}
hence
\begin{equation}
\nabla_\theta J(\theta) = \mathbb{E}_{\tilde{\tau} \sim p_\theta} \left[ \sum_{t=0}^{T-1} \nabla_\theta \log \pi_\theta(a_t \mid s_t) \cdot R_{\text{tier}}(\tilde{\tau}) \right].
\label{eq:pg_stochastic}
\end{equation}

\textbf{Step 3: Baseline Subtraction.}
For any $b_t(s_t)$ not depending on $a_t$,
\begin{align}
&\mathbb{E}_{\tilde{\tau} \sim p_\theta} \left[ \sum_{t=0}^{T-1} \nabla_\theta \log \pi_\theta(a_t \mid s_t) \cdot b_t(s_t) \right] \\
&= \sum_{t=0}^{T-1} \mathbb{E}_{s_t} \left[ b_t(s_t) \cdot \mathbb{E}_{a_t \sim \pi_\theta(\cdot \mid s_t)} \left[ \nabla_\theta \log \pi_\theta(a_t \mid s_t) \right] \right] = 0,
\end{align}
so
\begin{equation}
\mathbb{E}_{\tilde{\tau} \sim p_\theta} \left[ \sum_{t=0}^{T-1} \nabla_\theta \log \pi_\theta(a_t \mid s_t) \cdot (R_{\text{tier}}(\tilde{\tau}) - b_t(s_t)) \right] = \nabla_\theta J(\theta).
\label{eq:pg_baseline_stochastic}
\end{equation}
\end{proof}

\begin{remark}[If the reasoner is $\theta$-dependent]
If, instead, $p(\xi\mid\tau)$ depends on $\theta$ (e.g., shared trunk),
then
\[
\nabla_\theta\log p_\theta(\tilde{\tau})
=\sum_t \nabla_\theta\log\pi_\theta(a_t\mid s_t)\;+\;\nabla_\theta\log p_\theta(\xi\mid\tau),
\]
and an unbiased estimator must add the extra score term
$\nabla_\theta\log p_\theta(\xi\mid\tau)$ multiplied by the same return.
Alternatively, one may stop gradients through the reasoner or generate $\xi$ using frozen modules to enforce $\theta$-independence.
\end{remark}

\begin{proposition}[Unbiased gradient with $\theta$-dependent reasoner]
If $p_\theta(\xi\mid\tau)$ depends on $\theta$, then
\[
\nabla_\theta J(\theta)
=\mathbb{E}\!\left[\left(\sum_{t}\nabla_\theta\log\pi_\theta(a_t\mid s_t)
+\nabla_\theta\log p_\theta(\xi\mid\tau)\right)\,R_{\text{tier}}(\tilde{\tau})\right],
\]
so the unbiased score-function estimator must include both terms (each may use an appropriate baseline that is independent of the respective sampled variable).
\end{proposition}

\begin{corollary}
Under the $\theta$-independence of $p(\xi\mid\tau)$, the three-tier reward preserves the unbiasedness of the REINFORCE estimator for $\nabla_\theta J(\theta)$.
Algorithms such as PPO, GRPO, and BNPO, which optimize clipped or normalized surrogate objectives, remain applicable with this reward; however, their gradient estimates are generally biased (by design) and converge to stationary points of their respective surrogate objectives rather than guaranteeing an unbiased gradient of $J(\theta)$.
\end{corollary}

\section{Generation Diversity During Training}
\label{app:diversity}

An interesting emergent property of AC-RL training is that it maintains greater generation diversity compared to standard binary-reward RL. Figure~\ref{fig:diversity} tracks the fraction of training batches where all $M$ generated captions receive identical rewards (zero standard deviation), which serves as an indicator of diversity collapse.

\begin{figure}[h]
    \centering
    \includegraphics[width=0.67\linewidth]{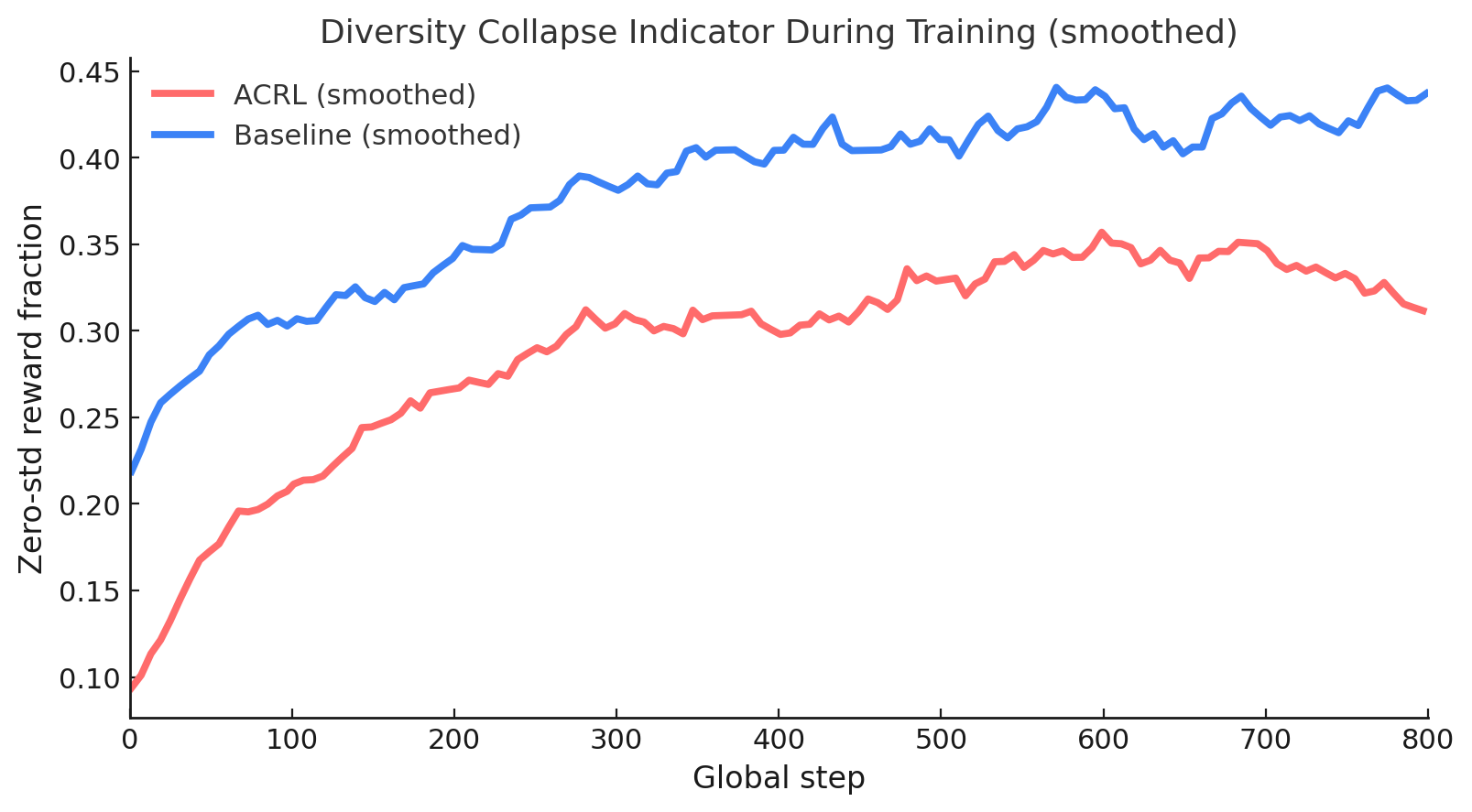}
    \caption{Fraction of uniform-reward batches during training. AC-RL (red) maintains lower values than standard RL (blue), indicating more diverse caption generation throughout training. Both methods show increasing trends as policies converge, but AC-RL's tiered reward structure preserves more exploration.}
    \label{fig:diversity}
\end{figure}

Both methods show an upward trend as entropy naturally decreases during policy optimization. However, AC-RL consistently maintains a lower fraction of uniform-reward batches (approximately 0.31 vs 0.42 at convergence). This difference likely stems from the tiered reward structure: while standard RL only distinguishes between success and failure, AC-RL's intermediate reward ($\alpha=0.7$) creates a richer gradient landscape that encourages the model to explore different captioning strategies.
\section{Prompt Templates}
\label{app:prompts}

We present the complete set of prompts used in our AC-RL framework. The prompts are structured to maintain clear role separation between the captioner (visual description) and reasoner (problem-solving), while enabling controlled interaction during training.

\subsection{Vision-Language Model Prompts}

\paragraph{Initial Caption Generation.}
The following prompt instructs the VLM to generate comprehensive visual descriptions without solving the problem:

\begin{tcolorbox}[colback=gray!5!white,colframe=gray!75!black,title=\texttt{vlm\_initial\_description\_prompt}]
\small
\texttt{I need your help analyzing this image to prepare for answering the following question:}

\texttt{\{question\}}

\texttt{IMPORTANT: DO NOT answer the question directly. Instead, provide a comprehensive and detailed description of everything visible in the image that could be relevant for answering this question.}

\texttt{Focus on describing:}
\begin{itemize}
\item \texttt{All objects, people, text, and visual elements in the image}
\item \texttt{Spatial relationships between different elements}
\item \texttt{Any text content that is visible, transcribed exactly}
\item \texttt{Colors, shapes, patterns, and visual attributes}
\item \texttt{Relevant contextual details and background information}
\end{itemize}

\texttt{Your description should be detailed enough that someone could mentally reconstruct the image without seeing it, but DO NOT provide step-by-step instructions on how to recreate it.}
\end{tcolorbox}

\paragraph{Clarification Response.}
When the reasoner requests specific visual information, the frozen reference model uses this prompt:

\begin{tcolorbox}[colback=gray!5!white,colframe=gray!75!black,title=\texttt{vlm\_focused\_description\_prompt}]
\small
\texttt{Original Question: \{question\}}

\texttt{Previous Description: \{previous\_descriptions\}}

\texttt{CONTEXT: The description above was provided for this image, but some details might be missing or unclear. We are asking this specific follow-up question to gather additional visual details.}

\texttt{Your specific task: \{focus\_request\}}

\texttt{CRITICAL INSTRUCTIONS:}
\begin{itemize}
\item \texttt{You are a VISUAL DESCRIBER only - DO NOT attempt to answer the original question}
\item \texttt{DO NOT solve the problem or provide calculations}
\item \texttt{DO NOT give step-by-step solutions or reasoning}
\item \texttt{ONLY describe what you can see in the image that relates to the specific request}
\item \texttt{Focus solely on visual elements: objects, text, numbers, shapes, spatial relationships}
\item \texttt{If asked about measurements, describe what you see but don't calculate or solve}
\item \texttt{If asked about equations, transcribe what's visible but don't solve them}
\item \texttt{Be thorough and precise in your description since this is to clarify specific missing details}
\end{itemize}
\end{tcolorbox}

\newpage
\subsection{Reasoner Prompts}

\paragraph{Adaptive Decision Mechanism.}
The reasoner evaluates whether the initial caption is sufficient or requires clarification:

\begin{tcolorbox}[colback=gray!5!white,colframe=gray!75!black,title=\texttt{reasoner\_adaptive\_decision\_prompt}]
\small
\texttt{You are an expert visual reasoning assistant. Your task is to analyze the given image description and decide if you can solve the problem directly or if you need one specific piece of additional visual information.}

\texttt{Image Description: \{description\}}

\texttt{Question: \{question\}}

\texttt{ANALYSIS INSTRUCTIONS:}
\begin{enumerate}
\item \texttt{\textbf{CAREFUL EVALUATION}: Analyze if the description contains all specific visual details needed to solve completely and accurately.}
\item \texttt{\textbf{BE CONSERVATIVE}: If missing ANY crucial visual detail, request MORE information rather than guess.}
\item \texttt{\textbf{ONE CLARIFICATION ONLY}: You can request specific additional visual information if needed.}
\item \texttt{\textbf{DECISION CRITERIA}:}
\begin{itemize}
\item \texttt{If you have ALL visual details needed: Status = SOLVED}
\item \texttt{If missing crucial visual information: Status = NEED\_MORE\_INFO}
\end{itemize}
\item \texttt{\textbf{AVOID ASSUMPTIONS}: Don't guess numbers, assume "typical" values, or fill in missing details.}
\end{enumerate}

\texttt{CRITICAL PRINCIPLES:}
\begin{itemize}
\item \texttt{\textbf{BE SPECIFIC in requests}: Ask for exact details you need}
\item \texttt{\textbf{SOLVE CONFIDENTLY when possible}: If you have enough information, provide the complete solution}
\item \texttt{\textbf{REQUEST STRATEGICALLY}: Make your one request count - ask for the most crucial missing details}
\end{itemize}

\texttt{OUTPUT FORMAT (all fields required):}\\
\texttt{Reasoning: [Your detailed analysis of what information you have and what might be missing]}\\
\texttt{Status: [SOLVED or NEED\_MORE\_INFO]}\\
\texttt{Answer: [Your complete final answer if Status is SOLVED - use $\backslash$boxed\{answer\} format, otherwise N/A]}\\
\texttt{Request: [Your specific request for additional visual information if Status is NEED\_MORE\_INFO, otherwise N/A]}
\end{tcolorbox}

\newpage
\paragraph{Final Answer Generation.}
For both direct solving and post-clarification scenarios:

\begin{tcolorbox}[colback=gray!5!white,colframe=gray!75!black,title=\texttt{reasoner\_final\_prompt}]
\small
\texttt{You are an expert mathematical reasoning assistant. Based on the complete image description below, please solve the mathematical problem step-by-step.}

\texttt{Complete Image Description: \{description\}}

\texttt{Question: \{question\}}

\texttt{INSTRUCTIONS:}
\begin{enumerate}
\item \texttt{Analyze the complete image description carefully}
\item \texttt{Work through the problem step-by-step with clear mathematical reasoning}
\item \texttt{Show all calculations and logical steps}
\item \texttt{Provide your final answer in the required format}
\item \texttt{Use $\backslash$boxed\{answer\} notation. For multiple choice, use $\backslash$boxed\{letter\} format}
\end{enumerate}

\texttt{You MUST follow this format:}\\
\texttt{<think>}\\
\texttt{Your detailed reasoning and thought process here...}\\
\texttt{</think>}\\
\texttt{<answer> Final Answer: your final answer here </answer>}
\end{tcolorbox}

\section{Training Hyperparameters}
\label{app:hyperparameters}

We provide complete hyperparameter specifications to ensure reproducibility. All experiments use the same random seed for dataset sampling to enable fair comparisons.

\begin{table}[h]
\centering
\caption{Hyperparameter settings for AC-RL training across different model sizes.}
\vspace{0.25cm}
\begin{tabular}{@{}lcc@{}}
\toprule
\textbf{Hyperparameter} & \textbf{2B Models} & \textbf{3B Models} \\
\midrule
\multicolumn{3}{@{}l}{\textit{Optimization}} \\
Learning rate & $3 \times 10^{-6}$ & $2 \times 10^{-6}$ \\
Effective batch size & 256 & 256 \\
KL divergence weight ($\beta$) & 0.001 & 0.001 \\
\midrule
\multicolumn{3}{@{}l}{\textit{LoRA Configuration}} \\
LoRA rank ($r$) & 128 & 256 \\
LoRA alpha ($\alpha$) & 256 & 512 \\
LoRA dropout & 0.05 & 0.05 \\
\midrule
\multicolumn{3}{@{}l}{\textit{BNPO Settings}} \\
Group size ($M$) & 8 & 8 \\
Number of iterations & 6 & 6 \\
Remaining parameters & \multicolumn{2}{c}{TRL library defaults} \\
\midrule
\multicolumn{3}{@{}l}{\textit{Generation Parameters}} \\
Captioner temperature & \multicolumn{2}{c}{1.0} \\
Captioner max tokens & \multicolumn{2}{c}{800} \\
Reasoner temperature & \multicolumn{2}{c}{0.6} \\
Reasoner top-$p$ & \multicolumn{2}{c}{0.95} \\
Reasoner max tokens & \multicolumn{2}{c}{100,000} \\
\midrule
\multicolumn{3}{@{}l}{\textit{Reward Configuration}} \\
Clarification penalty ($1-\alpha$) & \multicolumn{2}{c}{0.3} \\
\bottomrule
\end{tabular}
\end{table}

\paragraph{Implementation Details.}
We implement AC-RL using the Transformers Reinforcement Learning (TRL) library~\cite{TRL} with Beta-Normalization Policy Optimization (BNPO). The LoRA~\cite{lora} adapters are applied to all linear layers in the vision-language models. Training typically converges within 1,000 steps on the ViRL-39K dataset. All experiments use mixed precision training (fp16).

\paragraph{Computational Requirements.}
Training a 3B parameter captioner with AC-RL requires approximately 50 hours on 8 NVIDIA A6000-Ada GPUs. The 2B models require 40 hours on the same hardware configuration. The reasoner is hosted on a 4$\times$ AMD MI250X node, though it is never fully saturated during training.

\section{LLM Usage Statement}
\label{app:llm_usage}
We used large language models for grammatical corrections and rewording suggestions to improve clarity. All research ideas, experimental design, analysis, and scientific contributions are the original work of the authors. LLMs were not used for generating research content or results interpretation.
\end{document}